\documentclass{article}

\usepackage[preprint]{neurips_2026}

\usepackage[utf8]{inputenc}
\usepackage[T1]{fontenc}

\usepackage{amsmath,amssymb,amsfonts,bm,amsthm}

\usepackage{graphicx}
\usepackage{subfigure} 
\usepackage{wrapfig}
\usepackage{stfloats}

\usepackage{booktabs}
\usepackage{array}
\usepackage{multirow}
\usepackage{threeparttable}
\usepackage{colortbl}

\usepackage{algorithm}
\usepackage[noEnd=false,indLines=true]{algpseudocodex}

\usepackage{enumitem}
\usepackage{nicefrac}
\usepackage{textcomp}
\usepackage{verbatim}
\usepackage{comment}

\usepackage[table,dvipsnames]{xcolor}
\usepackage{pifont}

\usepackage{url}
\usepackage{hyperref}

\definecolor{rciColor}{RGB}{230,240,255}
\definecolor{gaiColor}{RGB}{235,247,237}
\definecolor{finalColor}{RGB}{255,243,230}
\definecolor{AlgoCodeBlue}{RGB}{45,82,115}
\definecolor{AlgoCommentBlue}{RGB}{30,75,125}
\definecolor{mypink}{rgb}{.99,.91,.95}

\newtheorem{theorem}{Theorem}[section]
\newtheorem{proposition}[theorem]{Proposition}

\newtheorem{definition}[theorem]{Definition}

\newcommand{\stdvu}[1]{\scriptsize{\color{darkgray}$(\pm{#1})$}\,{\color{ForestGreen}$\uparrow$}}
\newcommand{\stdvd}[1]{\scriptsize{\color{darkgray}$(\pm{#1})$}\,{\color{red}$\downarrow$}}
\newcommand{\stdvw}[1]{\scriptsize{\color{darkgray}$(\pm{#1})$}\,\hphantom{$\downarrow$}}

\title{FedSmoothLoRA: Toward Smoother and Faster Convergence in Federated Low-Rank Adaptation}

\author{
\textbf{Zehao Wang}$^{1}$ \quad
\textbf{Guanglei Yang}$^{1}$ \quad
\textbf{Yihan Zeng}$^{2}$ \quad
\textbf{Hang Xu}$^{2}$ \\
\textbf{Hongzhi Zhang}$^{1}$ \quad
\textbf{Wangmeng Zuo}$^{1}$ \quad
\textbf{Chun-Mei Feng}$^{3}$ \\
$^{1}$Harbin Institute of Technology \\
$^{2}$Huawei Noah's Ark Lab \\
$^{3}$University College Dublin
}
\begin{document}

\maketitle

\begin{abstract}
Federated fine-tuning of foundation models with Low-Rank Adaptation (LoRA) provides an efficient solution for reducing communication and computation costs while preserving data locality.
However, the direct combination of FedAvg and LoRA suffers from three key issues: limited update space, which restricts the model's effective learning capacity; inter-round state mismatch, which disrupts cross-round local optimization continuity; and a client-agnostic starting state, which slows local convergence on clients.
Although recent methods mitigate the limited update space issue by merging LoRA updates into the backbone across communication rounds, inter-round state mismatch and the client-agnostic starting state remain insufficiently addressed.
To address these issues, we propose \textbf{FedSmoothLoRA}, a federated LoRA tuning framework that preserves the enlarged update space, improves cross-round local optimization continuity, and provides a client-aware starting state for local training.
At each communication round, FedSmoothLoRA constructs the local LoRA initialization using two matrices: a Round-Matching matrix that preserves cross-round local state continuity, and a Gradient-Aligned matrix that provides client-specific optimization guidance from gradient signals estimated on local data.
Together, these designs enable smoother and faster convergence.
Extensive experiments on image classification and natural language generation tasks demonstrate that FedSmoothLoRA consistently outperforms existing federated LoRA tuning methods. \noindent\textbf{Code:} \url{https://github.com/wangzehao0704/FedSmoothLoRA}

\end{abstract}
\section{Introduction}
\label{sec:intro}

Foundation models have advanced rapidly across vision, language, and multimodal tasks, with representative examples including CLIP~\cite{radford2021learning}, LLaMA~\cite{touvron2023llama,touvron2023llama2,dubey2024llama3}, LLaVA~\cite{liu2024visual}, and Qwen~\cite{bai2023qwen}. 
However, fine-tuning these models on downstream or domain-specific data typically requires substantial computation, memory, and communication costs, while the data are often private and distributed across clients~\cite{achiam2023gpt,touvron2023llama,ye2024openfedllm}.
Federated Learning (FL)~\cite{mcmahan2017communication}, combined with Low-Rank Adaptation (LoRA)~\cite{hu2021lora}, provides a practical solution by enabling distributed fine-tuning with reduced computation and communication costs while preserving data locality.
A natural instantiation of this paradigm is to directly combine FedAvg~\cite{mcmahan2017communication} with LoRA, which we refer to as \textbf{FedAvgLoRA}~\cite{zhang2024towards,zhang2023fedpetuning,ye2024openfedllm,kuang2024federatedscope,fan2023fate}.
FedAvgLoRA freezes the backbone weights and restricts local optimization and communication to lightweight LoRA parameters.
As a result, it is substantially more efficient than full-model federated fine-tuning, making it practical for resource-constrained clients.

Despite its computational and communication efficiency, FedAvgLoRA still suffers from three key issues, as illustrated in Fig.~\ref{fig:teaser}(a).
\ding{182} \textit{Limited update space}: its inherently low-rank update structure restricts the effective parameter update space, thereby limiting the model's effective learning capacity.
\ding{183} \textit{Inter-round state mismatch}: 
at the beginning of each communication round, the local LoRA parameters learned by the client in the previous round are replaced by the aggregated LoRA parameters downloaded from the server.
This disrupts cross-round local optimization continuity, leading to abrupt parameter shifts and severe \emph{inter-round oscillations} across rounds (see Fig.~\ref{fig:teaser}(c)).
\ding{184} \textit{Client-agnostic starting state}: 
in every round, clients initialize their local LoRA parameters from the same global LoRA parameters, without explicitly incorporating client-specific optimization signals.
As a result, the local starting state can be poorly aligned with heterogeneous local objectives, reducing the effectiveness of subsequent local updates and slowing client-side convergence (see Fig.~\ref{fig:teaser}(c)).

To tackle \ding{182}, the recently proposed FRLoRA~\cite{yanfederated} merges the aggregated LoRA update into the backbone across communication rounds, thereby enlarging the effective parameter update space.
As a result, FRLoRA allows LoRA updates to accumulate in the full parameter space and generally achieves better performance than FedAvgLoRA.
However, as shown in Fig.~\ref{fig:teaser}(b), the latter two issues remain insufficiently addressed.
First, FRLoRA still starts each new round from the global LoRA parameters, so the local LoRA parameters learned by a client in the previous round are overwritten.
This disrupts cross-round local optimization continuity and leads to unstable cross-round optimization.
Second, although FRLoRA's model-weight-based initialization provides a more structured starting point than the shared LoRA initialization in FedAvgLoRA, it is derived from shared model weights rather than optimization signals estimated from local data.
Therefore, the starting point remains client-agnostic and may be poorly aligned with heterogeneous local objectives.

In this paper, we propose \textbf{FedSmoothLoRA}, as illustrated in Fig.~\ref{fig:method}.
Similar to FRLoRA, FedSmoothLoRA preserves the enlarged update space by merging LoRA updates into the backbone across communication rounds.
Beyond this, FedSmoothLoRA addresses \ding{183} and \ding{184} by constructing a local LoRA initialization for each client.
Specifically, this initialization consists of two matrices: a \textbf{Round-Matching} matrix and a Gradient-Aligned matrix.
The Round-Matching matrix preserves cross-round local state continuity by matching the current local starting state to the effective local model reached in the previous round, thereby reducing abrupt parameter shifts, mitigating inter-round oscillations, and promoting smoother cross-round optimization.
Inspired by LoRA-GA~\cite{wang2024loraga}, the Gradient-Aligned matrix adapts gradient-based initialization to the federated setting by estimating gradient signals from local data, thereby providing a client-aware starting state that better matches heterogeneous local objectives and accelerates local convergence.
Thus, FedSmoothLoRA achieves consistently stronger empirical performance than FedAvgLoRA and other existing federated LoRA tuning methods across image classification and natural language generation tasks.
\begin{figure*}[t]
    \centering
    \includegraphics[width=\textwidth]{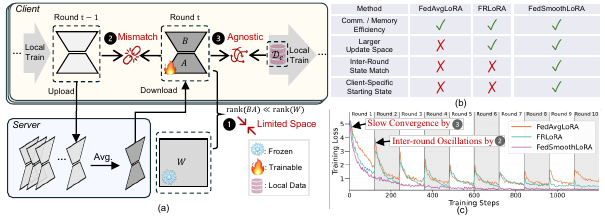}
    \vspace{-1.8em}
\caption{
\textbf{(a)} Illustration of the \textbf{three key issues} of FedAvgLoRA: 
\emph{Limited update space}, \emph{inter-round state mismatch}, 
and \emph{client-agnostic starting state}.
\textbf{(b)} \textbf{Comparison} of representative FL-LoRA methods. 
FedSmoothLoRA satisfies all three desired properties.
\textbf{(c)} \textbf{Training loss curves} across communication rounds.
FedSmoothLoRA achieves smoother training dynamics and faster local convergence compared with FedAvgLoRA and FRLoRA.
}
    \label{fig:teaser}
\end{figure*}

\textbf{Contribution.} In this work,
we propose \textbf{FedSmoothLoRA}, a federated LoRA fine-tuning framework that preserves the enlarged update space while improving cross-round local optimization continuity and providing a client-aware starting state for local training.
To mitigate inter-round state mismatch, we introduce the Round-Matching matrix, which matches the current local starting state to the effective local model reached in the previous round, thereby reducing abrupt parameter shifts and inter-round oscillations.
To address the client-agnostic starting state, we adapt gradient-based initialization to the federated setting by estimating gradient signals from local data, providing a client-aware starting state that better matches heterogeneous local objectives.
Extensive experiments on image classification and natural language generation tasks demonstrate that FedSmoothLoRA consistently outperforms existing federated LoRA tuning methods.

\section{Related Work}
\textbf{Parameter-Efficient Fine-tuning.}
Recently, PEFT (Parameter-Efficient Fine-Tuning)~\cite{han2024parameter} has emerged as an important technique for training models. 
Instead of fine-tuning all parameters of the model, PEFT introduces only a small number of trainable parameters at specific locations within the model~\cite{zaken2021bitfit, chen2022adaptformer,houlsby2019parameter,kim2021adapt}. 
Among these methods, LoRA~\cite{hu2021lora} (Low-Rank Adaptation) has become popular because it achieves performance on par with or better than previous PEFT methods without adding extra inference latency.
Several techniques have been proposed to enhance LoRA's structure, including AdaLoRA~\cite{zhang2023adalora}, 
ReLoRA~\cite{lialin2023relora},
rsLoRA~\cite{kalajdzievski2023rank}, 
DoRA~\cite{liu2024dora}, PiSSA~\cite{meng2024pissa}, and LoRA-GA~\cite{wang2024loraga}, among others~\cite{wang2024lorapro}. 
However, most of these methods do not account for the federated learning scenario, which can significantly degrade their performance. 
This motivates us to investigate how LoRA can be improved for better performance in federated learning settings.

\textbf{Federated Learning.}
Federated Learning (FL) has emerged as a widely adopted distributed solution for training models across decentralized data sources~\cite{mcmahan2017communication}.
Initially proposed in FedAvg~\cite{mcmahan2017communication}, FL achieved only moderate performance, especially in scenarios involving data heterogeneity. 
To address these limitations, various FL algorithms have been developed to improve performance, including FedAvgM~\cite{hsu2019measuring}, SCAFFOLD~\cite{karimireddy2020scaffold}, FedProx~\cite{li2020federated}, FedOPT~\cite{reddi2020adaptive}, and others.
Recently, methods such as FedLR~\cite{zhang2023fedpetuning} and FedIT~\cite{zhang2024towards} have explored combining FL with LoRA to improve federated tuning in communication- and computation-constrained environments.
OpenFedLLM~\cite{ye2024openfedllm} further reported that traditional FL methods do not achieve significant gains when directly integrated with LoRA.
Several LoRA-based FL methods, including FedRA~\cite{su2025fedra}, FlexLoRA~\cite{bai2024federated}, FFALoRA~\cite{sun2024improving}, and others~\cite{wang2024flora,singhal2025fedex,babakniya2023slora,yan2024federa,peng2026rethinking,fang2025federated}, improve federated LoRA tuning through aggregation, initialization, or optimization refinement.
FRLoRA~\cite{yanfederated} further enlarges the effective update space by merging LoRA updates into the backbone across rounds.
However, these methods still largely rely on a shared server-side LoRA or client-agnostic initialization at each round, leaving \textit{inter-round state mismatch} and the \textit{client-agnostic starting state} insufficiently addressed.
\section{Method}

\subsection{Preliminaries}
\label{sec:3.1}
\noindent\textbf{Low-Rank Adaptation (LoRA)}
~\cite{hu2021lora} is a widely used parameter-efficient fine-tuning method that substantially reduces the number of trainable parameters. It freezes the model weights and represents the update using two low-rank matrices, formulated as
\begin{flalign}
\boldsymbol{W}' = \boldsymbol{W} + \Delta \boldsymbol{W} := \boldsymbol{W} + s \boldsymbol{B}\boldsymbol{A},
\end{flalign}
where \(\boldsymbol{W} \in \mathbb{R}^{m \times n}\) denotes the original weight matrix, \(\boldsymbol{B} \in \mathbb{R}^{m \times r}\) and \(\boldsymbol{A} \in \mathbb{R}^{r \times n}\) are the trainable low-rank factors, and \(r \ll \min(m,n)\) is the LoRA rank. 
Following the scaling strategy of rsLoRA~\cite{kalajdzievski2023rank}, the scaling factor is set to \(s=\alpha/\sqrt{r}\), where \(\alpha\) is a scaling hyperparameter.
To ensure that the initial adapted weight remains consistent with the pre-trained weight, \(\boldsymbol{B}\) is initialized to zero, while \(\boldsymbol{A}\) is initialized using Kaiming initialization. 

\noindent\textbf{Rank-\(r\) SVD Approximation of a Matrix.}
The \(\operatorname{SVDApprox}\) operator used throughout the paper is defined as follows.

\begin{definition}[SVDApprox]
\label{rank-r_svd}
Let \(\boldsymbol{M} \in \mathbb{R}^{m \times n}\) have singular value decomposition \(\boldsymbol{M} = \bm{U}\boldsymbol{\Sigma}\bm{V}^{\top}\), 
where \(\bm{U} \in \mathbb{R}^{m \times m}\) and \(\bm{V} \in \mathbb{R}^{n \times n}\) are orthogonal matrices, and 
\(\boldsymbol{\Sigma} \in \mathbb{R}^{m \times n}\) is a rectangular diagonal matrix whose diagonal entries are the singular values of \(\boldsymbol{M}\), arranged in non-increasing order. 
For a target rank \(r \le \min(m,n)\), \(\operatorname{SVDApprox}\) is defined as
\begin{flalign}
\operatorname{SVDApprox}(\boldsymbol{M}; r) := (\boldsymbol{B}, \boldsymbol{A}),
\end{flalign}
where \(\boldsymbol{B} = \bm{U}_{[:, :r]} \boldsymbol{\Sigma}_{[:r, :r]}^{1/2} \in \mathbb{R}^{m \times r}\) and
\(\boldsymbol{A} = \boldsymbol{\Sigma}_{[:r, :r]}^{1/2} \left(\bm{V}_{[:, :r]}\right)^{\top} \in \mathbb{R}^{r \times n}\).
\end{definition}

The product \(\boldsymbol{B}\boldsymbol{A}\) is a truncated rank-\(r\) SVD approximation of \(\boldsymbol{M}\). By the Eckart--Young--Mirsky theorem~\cite{schmidt1907theorie}, it is the best rank-\(r\) approximation of \(\boldsymbol{M}\) under both the Frobenius norm and the spectral norm.

\noindent\textbf{FedAvgLoRA and FRLoRA.}
FedAvgLoRA~\cite{zhang2024towards,zhang2023fedpetuning,ye2024openfedllm} incorporates LoRA into the standard FedAvg framework for parameter-efficient federated fine-tuning. 
Given clients \(C=\{c\}_{c=1}^{K}\), each client \(c\) holds a local dataset \(\mathcal{D}_c\) with size \(n_c\), and all clients share the same pre-trained backbone \(\boldsymbol{W}_0\). 
At round \(t\), the server broadcasts the global LoRA parameters \((\boldsymbol{B}_s^t,\boldsymbol{A}_s^t)\), after which clients perform local training and upload the updated adapters \((\boldsymbol{B}_c^t,\boldsymbol{A}_c^t)\) for aggregation. 
Although FedAvgLoRA reduces communication and memory costs, it still suffers from limited update space, inter-round state mismatch, and a client-agnostic starting state. 
FRLoRA extends FedAvgLoRA by accumulating LoRA updates across rounds and merging the aggregated updates into the global backbone, thereby enlarging the effective update space, but the latter two issues remain unresolved.

\begin{figure*}[t]
    \centering
    \includegraphics[width=\linewidth]{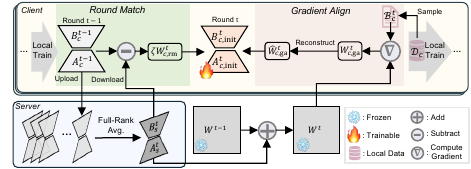}
\caption{
\textbf{Illustration} of the proposed \textbf{FedSmoothLoRA}. 
On the client side, local LoRA is initialized by combining the \textbf{Round-Matching} matrix
\(\boldsymbol{W}_{c,\mathrm{rm}}^{t}\) and the \textbf{Gradient-Aligned} matrix
\(\boldsymbol{\hat{W}}_{c,\mathrm{ga}}^{t}\), enabling smoother and faster local training.
On the server side, the uploaded client LoRA updates are aggregated in the full parameter space and merged into the backbone.
At each communication round, both clients and the server merge LoRA updates into the backbone, thereby accumulating updates across rounds and enlarging the update space.
}
    \label{fig:method}
\end{figure*}

\begin{figure*}[t]
\resizebox{1.0\linewidth}{!}{
\begin{minipage}{1.0\linewidth}
\begin{algorithm}[H]
\caption{FedSmoothLoRA}
\label{alg:2}
\begin{algorithmic}[1]

\State \textbf{Input:} Pretrained backbone \(\boldsymbol{W}\); clients \(C=\{c\}_{c=1}^{K}\) with local datasets \(\mathcal{D}_c\) and sizes \(n_c\); rounds \(T\); LoRA rank \(r\); scaling factor \(\alpha\); stabilization factor \(\gamma\); coefficient mode \(\zeta_{\mathrm{mode}}\)
\State \textbf{Output:} Final server backbone \(\boldsymbol{W}_s^{T}\)
\State \textbf{Initialization:} \(\boldsymbol{W}_s^{0}\leftarrow\boldsymbol{W}\), \(\boldsymbol{W}_c^{0}\leftarrow\boldsymbol{W}\ \forall c\in C\), \(\boldsymbol{A}_s^{0}\leftarrow\operatorname{KaimingInit}\), \(\boldsymbol{B}_s^{0}\leftarrow\boldsymbol{0}\), \(s\leftarrow\alpha/\sqrt{r}\)

\State \textbf{LocalUpdate}\((c,\boldsymbol{B}_s^t,\boldsymbol{A}_s^t,t)\):
\If{\(t=0\)}
\State \(\boldsymbol{W}_c^{t}\leftarrow\boldsymbol{W}_c^{0}\), \quad
\(\boldsymbol{W}^{t}_{c,\mathrm{rm}}\leftarrow\boldsymbol{0}\)
\Else
    \State \(\boldsymbol{W}_c^{t}\leftarrow\boldsymbol{W}_c^{t-1}+s\boldsymbol{B}_s^{t}\boldsymbol{A}_s^{t}\)
\State
\begin{minipage}[t]{0.57\linewidth}
\(\textcolor{red!70!black}{
\boldsymbol{W}^{t}_{c,\mathrm{rm}}\leftarrow
\boldsymbol{B}_{c}^{t-1}\boldsymbol{A}_{c}^{t-1}
-
\boldsymbol{B}_s^{t}\boldsymbol{A}_s^{t}
}\)
\end{minipage}
\hfill
\begin{minipage}[t]{0.42\linewidth}
\raggedleft
\(\textcolor{red!70!black}{\triangleright}\)
{\small \textcolor{red!70!black}{\textit{Build the Round-Matching matrix}}}
\end{minipage}

\EndIf
\State Sample a small mini-batch \(\mathcal{B}_c^t\subset\mathcal{D}_c\)

\State
\begin{minipage}[t]{0.58\linewidth}
\(\textcolor{red!70!black}{
\boldsymbol{W}^{t}_{c,\mathrm{ga}}\leftarrow
\nabla_{\boldsymbol{W}_c^t}\mathcal{L}(\mathcal{B}_c^t)
}\)
\end{minipage}
\hfill
\begin{minipage}[t]{0.42\linewidth}
\raggedleft
\(\textcolor{red!70!black}{\triangleright}\)
{\small \textcolor{red!70!black}{\textit{Build the Gradient-Aligned matrix}}}
\end{minipage}
\State \(\bm{U},\boldsymbol{\Sigma},\bm{V}\leftarrow \operatorname{SVD}(\boldsymbol{W}^{t}_{c,\mathrm{ga}})\), then reconstruct \(\hat{\boldsymbol{W}}^{t}_{c,\mathrm{ga}}\) by Eq.~\ref{eq:gai_recon}
\State \(\zeta\leftarrow\operatorname{CosineSchedule}(t,T)\) if \(\zeta_{\mathrm{mode}}=\texttt{decay}\); otherwise \(\zeta\leftarrow1\)
\State \(\boldsymbol{W}_{c,\mathrm{init}}^t
\leftarrow
\hat{\boldsymbol{W}}^{t}_{c,\mathrm{ga}}
+
\zeta\boldsymbol{W}^{t}_{c,\mathrm{rm}}\)

\State \((\boldsymbol{B}^{t}_{c,\mathrm{init}},\boldsymbol{A}^{t}_{c,\mathrm{init}})
\leftarrow
\operatorname{SVDApprox}\!\left(\boldsymbol{W}_{c,\mathrm{init}}^t;\,r\right)\)
\Comment{\small Local LoRA initialization}
\State Train on \(\mathcal{D}_c\) from \((\boldsymbol{B}^{t}_{c,\mathrm{init}},\boldsymbol{A}^{t}_{c,\mathrm{init}})\) with backbone \(\boldsymbol{W}_c^t-s\hat{\boldsymbol{W}}^{t}_{c,\mathrm{ga}}\), obtaining \((\widetilde{\boldsymbol{B}}_c^{t},\widetilde{\boldsymbol{A}}_c^{t})\)
\State \((\boldsymbol{B}_c^{t},\boldsymbol{A}_c^{t})\leftarrow\operatorname{SVDApprox}\!\left(\widetilde{\boldsymbol{B}}_c^{t}\widetilde{\boldsymbol{A}}_c^{t}-\hat{\boldsymbol{W}}^{t}_{c,\mathrm{ga}};\,r\right)\)
\State \textbf{return} \((\boldsymbol{B}_c^{t},\boldsymbol{A}_c^{t})\)

\Statex

\State \textbf{// Server Update}
\For{\(t=0,1,\dots,T-1\)}
    \ForAll{\(c\in C\) \textit{in parallel}}
        \State \((\boldsymbol{B}_c^{t},\boldsymbol{A}_c^{t})\leftarrow\textsc{LocalUpdate}(c,\boldsymbol{B}_s^t,\boldsymbol{A}_s^t,t)\)
    \EndFor
\State \(\Delta\boldsymbol{W}_s^{t+1}\leftarrow\frac{1}{N}\sum_{c\in C}n_c\,\boldsymbol{B}_c^{t}\boldsymbol{A}_c^{t}\)
\State \((\boldsymbol{B}_s^{t+1},\boldsymbol{A}_s^{t+1})\leftarrow\operatorname{SVDApprox}(\Delta\boldsymbol{W}_s^{t+1};r)\) \Comment{\small Server Update via Full-Rank Aggregation.}
\State \(\boldsymbol{W}_s^{t+1}\leftarrow\boldsymbol{W}_s^{t}+s\boldsymbol{B}_s^{t+1}\boldsymbol{A}_s^{t+1}\)
\EndFor

\State \textbf{return} \(\boldsymbol{W}_s^T\)

\end{algorithmic}
\end{algorithm}
\end{minipage}}
\vspace{-1.0em}
\end{figure*}

\subsection{FedSmoothLoRA}

To preserve the enlarged update space and further address inter-round state mismatch and the client-agnostic starting state, we propose FedSmoothLoRA.
As shown in Algorithm~\ref{alg:2} and Fig.~\ref{fig:method}, FedSmoothLoRA follows the standard federated optimization pipeline, consisting of a client-side local update and a server-side aggregation.
For clarity, we present the overall framework under the full device participation setting, where all clients participate in each communication round.

\subsubsection{Local Update}
At communication round \(t\), similar to FRLoRA, FedSmoothLoRA first merges the server-side aggregated LoRA update into the client backbone:
\begin{flalign}
\boldsymbol{W}_c^{t} \leftarrow \boldsymbol{W}_c^{t-1} + s \boldsymbol{B}_s^{t}\boldsymbol{A}_s^{t}.
\end{flalign}
This operation enlarges the model parameter update space, thereby addressing \ding{182}.

After updating the backbone, client \(c\) constructs a client-side initialization matrix \(\boldsymbol{W}_{c,\mathrm{init}}^t\) for local LoRA training. 
This initialization combines two matrices: the Round-Matching matrix \(\boldsymbol{W}_{c,\mathrm{rm}}^t\), which preserves cross-round local-state continuity, and the gradient-aligned matrix \(\hat{\boldsymbol{W}}_{c,\mathrm{ga}}^t\), which incorporates client-specific optimization signals. 

\noindent\textbf{Round-Matching Matrix.} FedSmoothLoRA reduces the discrepancy between the current server-side LoRA state and the client-specific local LoRA state learned in the previous round. 
For the first communication round, we set \(\boldsymbol{W}^{0}_{c,\mathrm{rm}}=\boldsymbol{0}\). 
For \(t>0\), the $\boldsymbol{W}_{c,\mathrm{rm}}^{t}$ term is defined as
\begin{flalign}
\boldsymbol{W}^{t}_{c,\mathrm{rm}} \leftarrow \boldsymbol{B}_{c}^{t-1}\boldsymbol{A}_{c}^{t-1} - \boldsymbol{B}_s^{t}\boldsymbol{A}_s^{t}.
\label{eq:rci}
\end{flalign}
This term captures the discrepancy between the LoRA state retained from the previous local round and the current server-side LoRA update. 
By incorporating it into the initialization, FedSmoothLoRA aligns the new round with the previous local training state, thereby suppressing abrupt parameter shifts and mitigating inter-round oscillations.
A detailed analysis of how $\boldsymbol{W}_{c,\mathrm{rm}}^{t}$ improves inter-round state consistency is provided in Appendix~\ref{app:inter_round_consistency}.

The Round-Matching matrix can be naturally extended to the partial participation setting. 
Let \(\hat{t}_c < t\) denote the most recent round before round \(t\) in which client \(c\) participated. 
When client \(c\) rejoins training at round \(t\), the corresponding Round-Matching term is defined as follows:
\begin{flalign}
\boldsymbol{W}_{c,\mathrm{rm}}^{t}
\leftarrow
\boldsymbol{B}_{c}^{\hat{t}_c} \boldsymbol{A}_{c}^{\hat{t}_c}
-
\sum_{\tau=\hat{t}_c+1}^{t} 
\boldsymbol{B}_{s}^{\tau} \boldsymbol{A}_{s}^{\tau}.
\label{eq:rm_partial}
\end{flalign}
This formulation matches the retained client-side LoRA state from the last active round with the accumulated server-side LoRA updates during the inactive period, thereby reducing state mismatch when the client re-enters training.

\noindent\textbf{Gradient-Aligned Matrix.} FedSmoothLoRA obtains a LoRA initialization that is more relevant to each client's local task data from local gradients rather than model weights.
Compared with the model-weight-based initialization in FRLoRA, which is shared across clients and is therefore less sensitive to client-specific objectives, $\boldsymbol{\hat{W}}_{c,\mathrm{ga}}^{t}$ constructs the initialization directly from each client's local optimization signal.
Specifically, for client \(c\) at round \(t\), to reduce memory consumption, we estimate the local gradient in a layer-wise manner using a small calibration mini-batch \(\mathcal{B}_c^t \subset \mathcal{D}_c\):
\begin{flalign}
\boldsymbol{W}_{c,\mathrm{ga}}^t \leftarrow 
\nabla_{\boldsymbol{W}_c^t}\mathcal{L}(\mathcal{B}_c^t),
\label{eq:gai_grad}
\end{flalign}
where \(\mathcal{B}_c^t\) is sampled from the client's local data.
This lightweight gradient estimate provides client-specific optimization signals, yielding a starting point that is better aligned with the client's local objective.

To further stabilize training, inspired by LoRA-GA~\cite{wang2024loraga}, we reconstruct \(\boldsymbol{W}_{c,\mathrm{ga}}^t\) as \(\hat{\boldsymbol{W}}_{c,\mathrm{ga}}^t\). 
Specifically, we first compute the singular value decomposition of \(\boldsymbol{W}_{c,\mathrm{ga}}^t\), i.e., \(\boldsymbol{W}_{c,\mathrm{ga}}^t = \bm{U} \boldsymbol{\Sigma} \bm{V}^{\top}\). 
We then reorganize the singular vectors to obtain a more stable gradient-aligned matrix:
\begin{flalign}
\hat{\boldsymbol{W}}_{c,\mathrm{ga}}^t \leftarrow \frac{\sqrt{d_{\mathrm{out}}}}{\gamma^2} \, \bm{U}_{[:,\,r+1:2r]} \left(\bm{V}_{[:,\,1:r]}\right)^{\top},
\label{eq:gai_recon}
\end{flalign}
where \(\gamma\) is a stabilization hyperparameter, \(d_{\mathrm{out}}\) denotes the output dimension of \(\boldsymbol{W}_{c,\mathrm{ga}}^t\), and \(2r \le \min(m,n)\).

\noindent\textbf{Local LoRA Initialization.} Based on the two terms above, FedSmoothLoRA constructs the final initialization matrix as
\begin{flalign}
\boldsymbol{W}_{c,\mathrm{init}}^t \leftarrow 
\hat{\boldsymbol{W}}_{c,\mathrm{ga}}^t 
+ \zeta \boldsymbol{W}^t_{c,\mathrm{rm}}.
\label{eq:init_matrix}
\end{flalign}
Here, \(\zeta\) controls the contribution of the Round-Matching term. 
FedSmoothLoRA supports two modes for \(\zeta\): a \textit{constant mode} and a \textit{decay mode}. 
In the \textit{constant mode}, we set \(\zeta=1\), which is used for the \texttt{IID} setting since clients tend to follow relatively consistent optimization directions and stronger cross-round continuity is beneficial. 
In the \textit{decay mode}, \(\zeta\) follows a cosine schedule, which is used for the \texttt{Non-IID} setting to preserve the stabilizing effect of $\boldsymbol{W}_{c,\mathrm{rm}}^{t}$ in early rounds while gradually reducing its influence to avoid over-constraining heterogeneous local updates.

The initialized LoRA parameters \(\boldsymbol{B}_{c,\mathrm{init}}^t\) and \(\boldsymbol{A}_{c,\mathrm{init}}^t\) are then obtained by
\begin{flalign}
(\boldsymbol{B}_{c,\mathrm{init}}^t, \boldsymbol{A}_{c,\mathrm{init}}^t)
\leftarrow
\operatorname{SVDApprox}(\boldsymbol{W}_{c,\mathrm{init}}^t;\, r).
\label{eq:init_general}
\end{flalign}

Starting from \(\boldsymbol{B}_{c,\mathrm{init}}^t\) and \(\boldsymbol{A}_{c,\mathrm{init}}^t\), client \(c\) performs local optimization on its local dataset \(\mathcal{D}_c\) with the backbone frozen as \(\boldsymbol{W}_c^{t} - s\hat{\boldsymbol{W}}_{c,\mathrm{ga}}^t\), and obtains temporary LoRA parameters \(\widetilde{\boldsymbol{B}}_c^t\) and \(\widetilde{\boldsymbol{A}}_c^t\). 
After local training, the effective LoRA update uploaded by client \(c\) is computed as
\begin{flalign}
(\boldsymbol{B}_c^t, \boldsymbol{A}_c^t) \leftarrow \operatorname{SVDApprox}\!\left(\widetilde{\boldsymbol{B}}_c^t\widetilde{\boldsymbol{A}}_c^t - \hat{\boldsymbol{W}}_{c,\mathrm{ga}}^t;\, r\right).
\label{eq:client_update}
\end{flalign}
Finally, \(\boldsymbol{B}_c^t\) and \(\boldsymbol{A}_c^t\) are transmitted back to the server for aggregation.
\subsubsection{Server Update via Full-Rank Aggregation}
After receiving the uploaded LoRA updates from all clients, the server first performs \emph{Full-Rank Aggregation} in the full update space, similar to~\cite{bai2024federated}:
\begin{flalign}
\Delta \boldsymbol{W}_s^{t+1} \leftarrow \frac{1}{N}\sum_{c \in C} n_c\, \boldsymbol{B}_c^{t}\boldsymbol{A}_c^{t},
\label{eq:server_delta}
\end{flalign}
where $N=\sum_{c\in C} n_c$. The aggregated update is then projected back to a rank-$r$ LoRA parameterization:
\begin{flalign}
(\boldsymbol{B}_s^{t+1}, \boldsymbol{A}_s^{t+1}) \leftarrow \operatorname{SVDApprox}(\Delta \boldsymbol{W}_s^{t+1}; r).
\label{eq:server_svd}
\end{flalign}
Compared with the \emph{Low-Rank Aggregation} strategy in FedAvgLoRA, which directly averages the low-rank factors \(B\) and \(A\), \emph{Full-Rank Aggregation} first integrates client updates in a less restricted matrix space and then projects the result back to a rank-$r$ LoRA form. This design allows the server to preserve more global update information before compression. 

Similar to the client side, the server further merges the aggregated LoRA update into the global backbone:
\begin{flalign}
\boldsymbol{W}_s^{t+1} \leftarrow \boldsymbol{W}_s^{t} + s \boldsymbol{B}_s^{t+1}\boldsymbol{A}_s^{t+1}.
\label{eq:server_backbone}
\end{flalign}

\section{Experiments}
\label{experiment}
\subsection{Image Classification Tasks}
\label{sec:4.1}

\noindent\textbf{Experimental setup.}
We evaluate all methods on CIFAR-100 with 5 clients under both IID and Non-IID partitions to verify whether the proposed design can mitigate the limitations of FedAvgLoRA. 
The 50,000 training samples are evenly distributed across clients, and each client further splits its local data into 80\% for training and 20\% for validation.
For the Non-IID setting, the label distribution is generated by a Dirichlet distribution with $\beta=0.1$.
We use ViT-Small~\cite{alexey2020image} pre-trained on ImageNet-21k as the backbone, with LoRA rank $r=2$ and scaling factor $\alpha=4$.

\begin{wraptable}{r}{0.50\textwidth}
    \vspace{-1.2em}
    \centering
    \caption{Top-1 test accuracy (\%) on CIFAR-100 under \texttt{Non-IID} and \texttt{IID} client partitions. The best results are highlighted in \textbf{bold}, and ${\color{ForestGreen}\uparrow}$ and ${\color{red}\downarrow}$ indicate increments and decrements compared with FedAvgLoRA.}
    \label{tab:cifar100_results}
    \resizebox{\linewidth}{!}{
    \begin{tabular}{lcc}
        \toprule
        \textbf{Method} & \textbf{\texttt{Non-IID}} & \textbf{\texttt{IID}} \\
        \midrule
        FedAvgLoRA
                                  & 57.36\stdvw{0.07} & 73.97\stdvw{0.11} \\
        FedAvgM(LoRA)\cite{hsu2019measuring}
                                  & 57.53\stdvu{0.17} & 75.37\stdvu{0.10} \\
        SCAFFOLD(LoRA)\cite{karimireddy2020scaffold}
                                  & 58.79\stdvu{0.12} & 75.71\stdvu{0.06} \\
        FedProx(LoRA)\cite{li2020federated}
                                  & 58.79\stdvu{0.13} & 74.35\stdvu{0.28} \\
        FlexLoRA\cite{bai2024federated}
                                  & 56.23\stdvd{0.63} & 74.08\stdvu{0.17} \\
        FFALoRA\cite{sun2024improving}
                                  & 10.27\stdvd{0.40} & 17.86\stdvd{0.20} \\
        LoRA-FAIR\cite{bian2025lora}
                                  & 59.03\stdvu{0.08} & 75.33\stdvu{0.15} \\
        FRLoRA\cite{yanfederated}
                                  & 60.28\stdvu{0.26} & 79.97\stdvu{0.16} \\
        {\cellcolor{mypink}\textbf{\texttt{FedSmoothLoRA}}}
                                  & {\cellcolor{mypink}\textbf{64.46}\stdvu{0.18}} & {\cellcolor{mypink}\textbf{84.07}\stdvu{0.13}} \\
        \bottomrule
    \end{tabular}
    }
    \vspace{-1.0em}
\end{wraptable}
\noindent\textbf{Results.}
Table~\ref{tab:cifar100_results} shows that \textbf{\texttt{FedSmoothLoRA}} achieves the best performance in both \texttt{Non-IID} and \texttt{IID} settings, reaching $64.46\%$ and $84.07\%$, respectively.
Compared with FedAvgLoRA, it improves the accuracy by $7.10$ and $10.10$ percentage points, respectively.
It also consistently outperforms the strongest baseline, FRLoRA, by $4.18$ and $4.10$ percentage points under the \texttt{Non-IID} and \texttt{IID} settings, respectively.
These results demonstrate that \textbf{\texttt{FedSmoothLoRA}} is effective under both heterogeneous and homogeneous client data distributions.
Similar to FRLoRA, FedSmoothLoRA benefits from a larger effective parameter update space, which helps alleviate the low-rank update-space limitation in \ding{182}.
More importantly, the additional gains over FRLoRA indicate that merely enlarging the update space is not sufficient for stable and efficient federated LoRA tuning.
The proposed \textbf{\texttt{$\boldsymbol{W}_{c,\mathrm{rm}}^{t}$}} mitigates \ding{183} by reducing inter-round state mismatch and preserving cross-round optimization continuity, while \textbf{\texttt{$\boldsymbol{\hat{W}}_{c,\mathrm{ga}}^{t}$}} addresses \ding{184} by providing a more task-relevant and client-aware initialization for local updates.
Together, these two components lead to smoother optimization and better final performance.

\subsection{Natural Language Generation Tasks}
\label{sec:4.2}

\quad\begin{wraptable}{r}{0.50\textwidth}
    \vspace{-1.9em}
    \centering
    \caption{Performance on \textbf{math and code tasks} under the \texttt{IID} setting. The backbone model is LLaMA-3.2-1B. The best results are highlighted in \textbf{bold}, and ${\color{ForestGreen}\uparrow}$ and ${\color{red}\downarrow}$ indicate increments and decrements compared with FedAvgLoRA.}
    \label{tab:math_code_results}
    \resizebox{\linewidth}{!}{
    \begin{tabular}{l|cc}
        \toprule
        \textbf{Method} & \textbf{GSM8K} & \textbf{HumanEval} \\
        \midrule
        FedAvgLoRA
        & 25.85\stdvw{0.23} & 15.81\stdvw{0.35} \\
        FedAvgM(LoRA)\cite{hsu2019measuring}
        & 25.70\stdvd{0.65} & 15.61\stdvd{0.24} \\
        SCAFFOLD(LoRA)\cite{karimireddy2020scaffold}
        & 26.18\stdvu{0.23} & 16.43\stdvu{0.39} \\
        FedProx(LoRA)\cite{li2020federated}
        & 25.47\stdvd{0.49} & 16.05\stdvu{0.19} \\
        FlexLoRA\cite{bai2024federated}
        & 26.47\stdvu{0.27} & 17.72\stdvu{0.19} \\
        FFALoRA\cite{sun2024improving}
        & 20.62\stdvd{0.73} & 15.20\stdvd{0.46} \\
        LoRA-FAIR\cite{bian2025lora} & 26.56\stdvu{0.27} & 17.48\stdvu{0.35} \\
        FRLoRA\cite{yanfederated}
        & 33.98\stdvu{0.09} & 18.12\stdvu{0.19} \\
        {\cellcolor{mypink}\textbf{\texttt{FedSmoothLoRA}}}
        & {\cellcolor{mypink}\textbf{36.74}\stdvu{0.51}}
        & {\cellcolor{mypink}\textbf{18.25}\stdvu{0.26}} \\
        \bottomrule
    \end{tabular}
    }
\end{wraptable}
\noindent\textbf{Experimental Setting.}
To evaluate the performance of FedSmoothLoRA on large language models, we conduct experiments on three natural language generation tasks using LLaMA-3.2-1B~\cite{dubey2024llama} as the backbone. 
Unless otherwise specified, all models are trained for 10 communication rounds with 200 local training steps per round, and the LoRA configuration is fixed to \(r=32\) and \(\alpha=64\).
For the math and code tasks, we adopt an IID setting with 3 clients, where the training data are evenly split across clients. For the math task, the model is trained on a 100k subset of MetaMathQA~\cite{yu2023metamath} and evaluated on GSM8K~\cite{cobbe2021training} using accuracy. For the code task, the model is trained on a 100k subset of Code-Feedback~\cite{zheng2024opencodeinterpreter} and evaluated on HumanEval~\cite{chen2021evaluating} using Pass@1. 
For the chat task, we use the Aya dataset~\cite{singh2024aya} and consider a language-skewed \texttt{Non-IID} setting, where the training data are partitioned across seven language clients. We evaluate the resulting model on English, Arabic, Russian, Chinese, Portuguese, French, and Spanish.

\noindent\textbf{Results.}
As shown in Tables~\ref{tab:math_code_results} and~\ref{tab:chat_results}, \textbf{\texttt{FedSmoothLoRA}} achieves the best overall performance across math reasoning, code generation, and multilingual chat tasks in federated LLM training.
Specifically, on GSM8K, \textbf{\texttt{FedSmoothLoRA}} reaches $36.74$, surpassing FedAvgLoRA by $10.89$ points and the strongest baseline, FRLoRA, by $2.76$ points.
On HumanEval, it achieves $18.25$ Pass@1, improving over FedAvgLoRA by $2.44$ points.
On the multilingual chat benchmark, \textbf{\texttt{FedSmoothLoRA}} obtains the best average score of $4.03$, compared with $3.54$ for FedAvgLoRA and $3.83$ for FRLoRA.
It also achieves the best results on several languages, including English, Chinese, French, and Spanish, while remaining competitive on Arabic and Portuguese.
These results suggest that the advantage of \textbf{\texttt{FedSmoothLoRA}} is not limited to a single task, but transfers consistently across reasoning, coding, and open-ended conversational generation, while also improving overall multilingual robustness under language-skewed \texttt{Non-IID} client distributions.
Similar to FRLoRA, \textbf{\texttt{FedSmoothLoRA}} benefits from enlarging the effective update space to alleviate \ding{182}.
More importantly, its gains over FRLoRA further support the effectiveness of \textbf{\texttt{$\boldsymbol{W}_{c,\mathrm{rm}}^{t}$}} and \textbf{\texttt{$\boldsymbol{\hat{W}}_{c,\mathrm{ga}}^{t}$}}.
Specifically, \textbf{\texttt{$\boldsymbol{W}_{c,\mathrm{rm}}^{t}$}} improves cross-round training continuity by reducing state inconsistency related to \ding{183}, while \textbf{\texttt{$\boldsymbol{\hat{W}}_{c,\mathrm{ga}}^{t}$}} provides a more task-aware initialization for local updates and thus alleviates \ding{184}.
\begin{table*}[t]
    \centering
    \caption{Performance on the chat task under the language-skewed \texttt{Non-IID} setting on the Aya dataset (English, Arabic, Russian, Chinese, Portuguese, French, and Spanish). The backbone model is LLaMA-3.2-1B. The best results are highlighted in \textbf{bold}, and ${\color{ForestGreen}\uparrow}$ and ${\color{red}\downarrow}$ indicate increments and decrements compared with FedAvgLoRA.}
    \label{tab:chat_results}
    \resizebox{\linewidth}{!}{
    \begin{tabular}{l|cccccccc}
        \toprule
        \textbf{Method} & \textbf{English} & \textbf{Arabic} & \textbf{Russian} & \textbf{Chinese} & \textbf{Portuguese} & \textbf{French} & \textbf{Spanish} & \textbf{Average} \\
        \midrule
        FedAvgLoRA
        & 4.33\stdvw{0.16}
        & 4.10\stdvw{0.58}
        & 3.20\stdvw{0.30}
        & 3.65\stdvw{0.26}
        & 2.73\stdvw{0.20}
        & 2.77\stdvw{0.25}
        & 4.00\stdvw{0.26}
        & 3.54\stdvw{0.09} \\
        
        SCAFFOLD(LoRA)\cite{karimireddy2020scaffold}
        & 4.32\stdvd{0.26}
        & 4.12\stdvu{0.38}
        & 3.25\stdvu{0.33}
        & 3.50\stdvd{0.38}
        & 2.72\stdvd{0.21}
        & 2.77\stdvw{0.26}
        & 4.02\stdvu{0.30}
        & 3.53\stdvd{0.06} \\
        
        FedProx(LoRA)\cite{li2020federated}
        & 4.23\stdvd{0.30}
        & 3.92\stdvd{0.46}
        & 3.03\stdvd{0.23}
        & 3.52\stdvd{0.31}
        & 2.62\stdvd{0.28}
        & 2.67\stdvd{0.30}
        & 3.85\stdvd{0.28}
        & 3.40\stdvd{0.11} \\
        
        FedAvgM(LoRA)\cite{hsu2019measuring}
        & 4.37\stdvu{0.23}
        & 4.05\stdvd{0.40}
        & 3.18\stdvd{0.35}
        & 3.60\stdvd{0.30}
        & 2.75\stdvu{0.20}
        & 2.80\stdvu{0.25}
        & 4.05\stdvu{0.30}
        & 3.54\stdvw{0.04} \\
        
        FlexLoRA\cite{bai2024federated}
        & 4.62\stdvu{0.28}
        & 4.02\stdvd{0.25}
        & 3.23\stdvu{0.30}
        & 3.68\stdvu{0.28}
        & 2.80\stdvu{0.15}
        & 2.88\stdvu{0.20}
        & 4.07\stdvu{0.18}
        & 3.61\stdvu{0.08} \\
        
        FFALoRA\cite{sun2024improving}
        & 4.40\stdvu{0.31}
        & 3.93\stdvd{0.38}
        & 3.10\stdvd{0.41}
        & 3.48\stdvd{0.37}
        & 2.80\stdvu{0.28}
        & 2.88\stdvu{0.26}
        & 4.03\stdvu{0.28}
        & 3.52\stdvd{0.03} \\
        LoRA-FAIR\cite{bian2025lora}
        & 4.70\stdvu{0.18}
        & 4.13\stdvu{0.20}
        & \textbf{3.53}\stdvu{0.33}
        & 3.60\stdvu{0.40}
        & 2.90\stdvu{0.13}
        & 3.00\stdvu{0.09}
        & 4.18\stdvu{0.15} 
        & 3.72\stdvu{0.14} \\
        FRLoRA\cite{yanfederated}
        & 5.18\stdvu{0.29}
        & \textbf{4.50}\stdvu{0.31}
        & 3.17\stdvd{0.26}
        & 3.65\stdvw{0.39}
        & \textbf{3.18}\stdvu{0.24}
        & 3.12\stdvu{0.21}
        & 3.98\stdvd{0.33}
        & 3.83\stdvu{0.14} \\
        
        {\cellcolor{mypink}\textbf{\texttt{FedSmoothLoRA}}}
        & {\cellcolor{mypink}\textbf{5.65}\stdvu{0.22}}
        & {\cellcolor{mypink}4.20\stdvu{0.26}}
        & {\cellcolor{mypink}3.35\stdvu{0.26}}
        & {\cellcolor{mypink}\textbf{4.30}\stdvu{0.22}}
        & {\cellcolor{mypink}2.75\stdvu{0.17}}
        & {\cellcolor{mypink}\textbf{3.45}\stdvu{0.28}}
        & {\cellcolor{mypink}\textbf{4.50}\stdvu{0.33}}
        & {\cellcolor{mypink}\textbf{4.03}\stdvu{0.06}} \\
        \bottomrule
    \end{tabular}
    }
\end{table*}

\begin{figure*}[t]
    \centering
    \begin{minipage}[t]{0.48\linewidth}
        \centering
        \includegraphics[width=\linewidth]{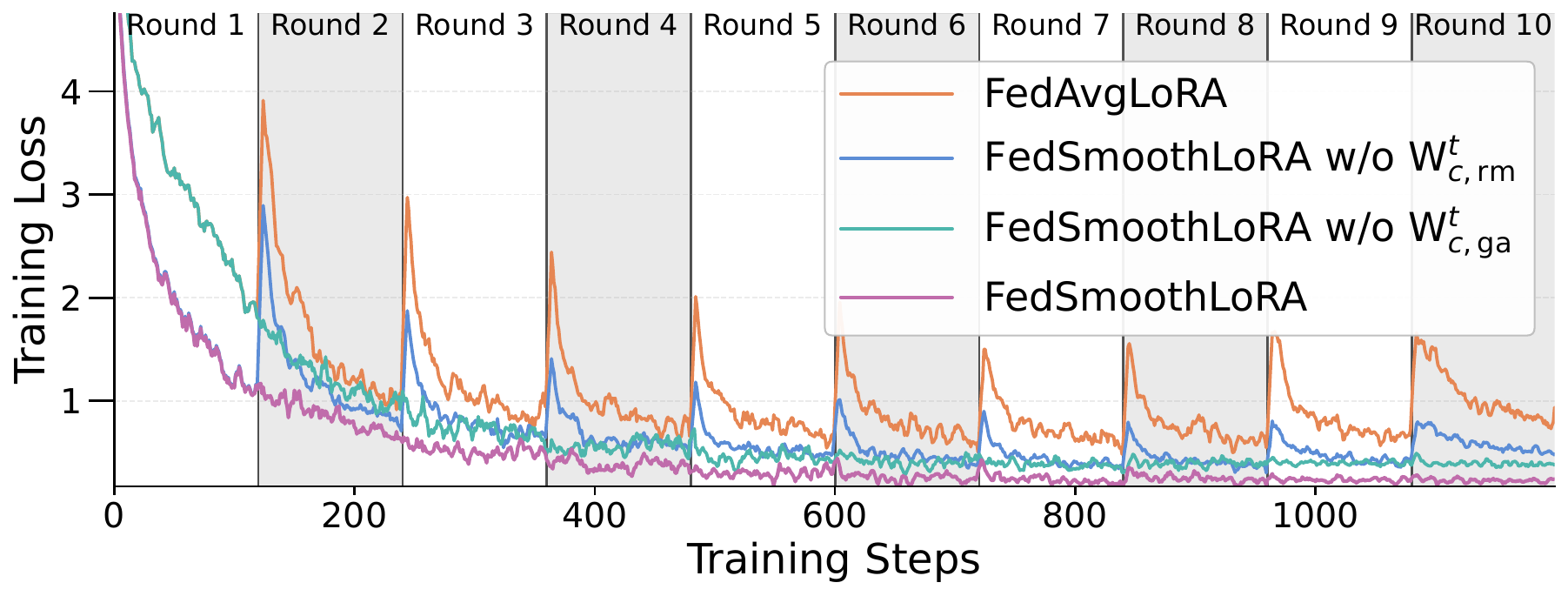}
        \centerline{\small (a) \texttt{Non-IID}}
    \end{minipage}
    \hfill
    \begin{minipage}[t]{0.48\linewidth}
        \centering
        \includegraphics[width=\linewidth]{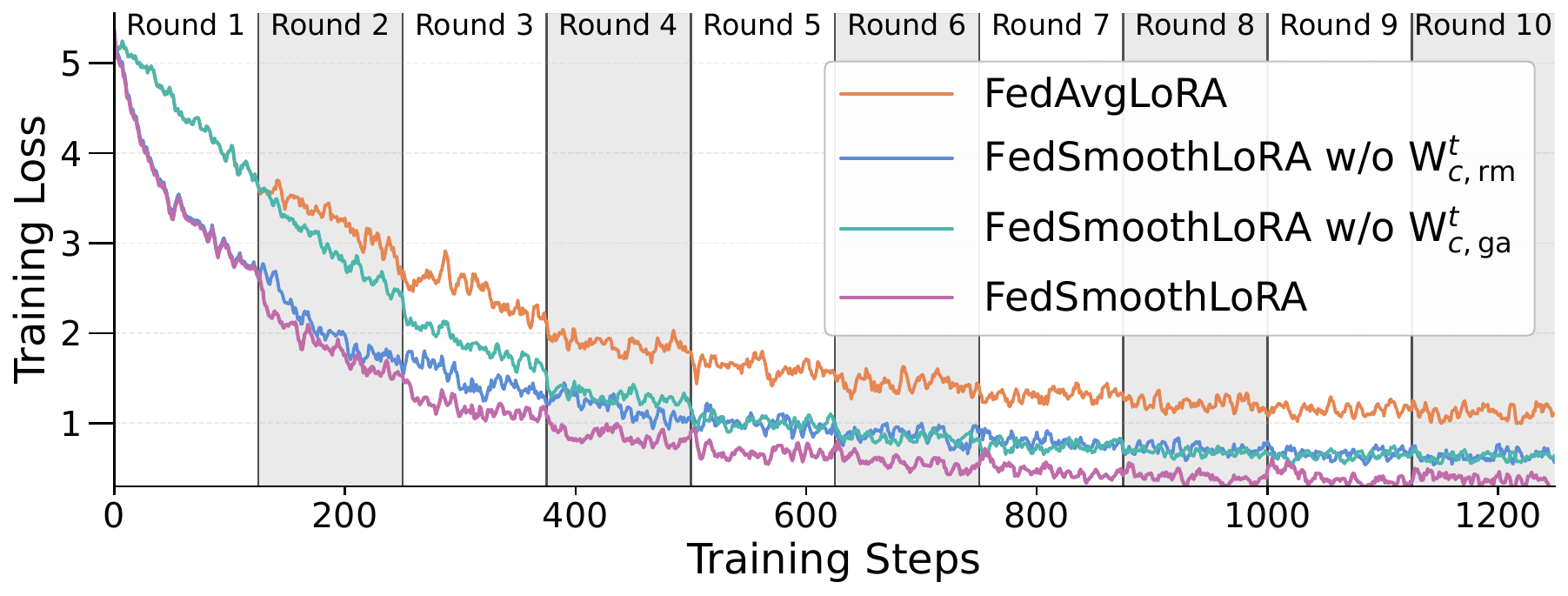}
        \centerline{\small (b) \texttt{IID}}
    \end{minipage}
\caption{Training loss curves of different variants under \texttt{Non-IID} and \texttt{IID} settings. The curves show that \textbf{\texttt{$\boldsymbol{W}_{c,\mathrm{rm}}^{t}$}} stabilizes cross-round optimization, while \textbf{\texttt{$\boldsymbol{\hat{W}}_{c,\mathrm{ga}}^{t}$}} provides a better initialization for local updates.}
    \label{fig:loss_curves}
    \vspace{-1.5em}
\end{figure*}

\subsection{Ablation Study}
\label{sec:4.3}
\quad \begin{wraptable}{r}{0.50\textwidth}
\vspace{-1.0em}
    \centering
    \caption{Ablation study on the proposed components on CIFAR-100 with the ViT-Small model under \texttt{Non-IID} and \texttt{IID} settings. The best results are highlighted in \textbf{bold}.}
    \label{tab:3}
    \resizebox{\linewidth}{!}{
    \begin{tabular}{l|cc}
    \toprule
    \textbf{Method} & \textbf{\texttt{Non-IID}} & \textbf{\texttt{IID}} \\
    \midrule
    FedAvgLoRA                 
    & 57.36\stdvw{0.07}
    & 73.97\stdvw{0.11} \\
    FedSmoothLoRA w/o $\boldsymbol{W}_{c,\mathrm{rm}}^{t}$      
    & 60.92\stdvw{0.17}
    & 81.78\stdvw{0.07} \\
    FedSmoothLoRA w/o $\boldsymbol{\hat{W}}_{c,\mathrm{ga}}^{t}$      
    & 61.79\stdvw{0.24}
    & 82.93\stdvw{0.32} \\
    {\cellcolor{mypink}\textbf{\texttt{FedSmoothLoRA}}}              
    & {\cellcolor{mypink}\textbf{64.46}\stdvw{0.18}}
    & {\cellcolor{mypink}\textbf{84.07}\stdvw{0.13}} \\
    \bottomrule
    \end{tabular}
    }
\end{wraptable}
\noindent{\textbf{Component Analysis of {\texttt{$\boldsymbol{W}_{c,\mathrm{rm}}^{t}$}} and {\texttt{$\boldsymbol{\hat{W}}_{c,\mathrm{ga}}^{t}$}}.}} 
The performance gain of \textbf{\texttt{FedSmoothLoRA}} comes not only from the enlarged effective update space, but also from the two client-side initialization components, namely, \textbf{\texttt{$\boldsymbol{W}_{c,\mathrm{rm}}^{t}$}} and \textbf{\texttt{$\boldsymbol{\hat{W}}_{c,\mathrm{ga}}^{t}$}}.
To verify their individual contributions, we conduct an ablation study in Table~\ref{tab:3} and provide the corresponding training loss curves in Fig.~\ref{fig:loss_curves}. 
Removing \textbf{\texttt{$\boldsymbol{W}_{c,\mathrm{rm}}^{t}$}} reduces the accuracy from $84.07$ to $81.78$ under the \texttt{IID} setting and from $64.46$ to $60.92$ under the \texttt{Non-IID} setting. 
As shown in Fig.~\ref{fig:loss_curves}(b), under \texttt{IID} data partitions, removing this term mainly slows down convergence at the beginning of each new communication round. 
In contrast, under the \texttt{Non-IID} setting in Fig.~\ref{fig:loss_curves}(a), the absence of \textbf{\texttt{$\boldsymbol{W}_{c,\mathrm{rm}}^{t}$}} leads to sharper loss fluctuations across rounds. 
These observations indicate that \textbf{\texttt{$\boldsymbol{W}_{c,\mathrm{rm}}^{t}$}} is important for preserving cross-round optimization continuity and mitigating the inter-round state mismatch in \ding{183}. 
Similarly, removing \textbf{\texttt{$\boldsymbol{\hat{W}}_{c,\mathrm{ga}}^{t}$}} decreases the accuracy from $84.07$ to $82.93$ under the \texttt{IID} setting and from $64.46$ to $61.79$ under the \texttt{Non-IID} setting. 
The corresponding loss curves show that, without this gradient-aligned component, the local loss decreases more slowly at the early stage of each round. 
This suggests that \textbf{\texttt{$\boldsymbol{\hat{W}}_{c,\mathrm{ga}}^{t}$}} provides a more effective starting point for local adaptation by incorporating client-specific optimization signals, thereby helping alleviate the client-agnostic starting-state issue in \ding{184}. 
When both components are enabled, \textbf{\texttt{FedSmoothLoRA}} achieves the best performance under both \texttt{Non-IID} and \texttt{IID} settings. 
These results suggest that the gains of \textbf{\texttt{FedSmoothLoRA}} arise from the joint effect of enlarging the effective update space to address \ding{182}, improving cross-round stability through \textbf{\texttt{$\boldsymbol{W}_{c,\mathrm{rm}}^{t}$}}, and enhancing client-aware local adaptation through \textbf{\texttt{$\boldsymbol{\hat{W}}_{c,\mathrm{ga}}^{t}$}}.

\begin{wraptable}{r}{0.50\textwidth}
    \vspace{-1.4em}
    \centering
    \caption{Comparison of different \(\zeta\) modes in \textbf{\texttt{FedSmoothLoRA}} on CIFAR-100 with the ViT-Small model under \texttt{Non-IID} and \texttt{IID} settings. FedAvgLoRA is reported as the baseline. The better results among \(\zeta\) modes are highlighted in \textbf{bold}.}
    \label{tab:6}
    \resizebox{\linewidth}{!}{
    \begin{tabular}{l|c|cc}
    \toprule
    \textbf{Method} & \textbf{\(\zeta\) Mode} & \textbf{\texttt{Non-IID}} & \textbf{\texttt{IID}} \\
    \midrule
    FedAvgLoRA 
    & -- 
    & 57.36\stdvw{0.07}
    & 73.97\stdvw{0.11} \\
    \midrule
    \multirow{2}{*}{\textbf{\texttt{FedSmoothLoRA}}}
    & \textit{Constant} 
    & 60.13\stdvw{0.64}
    & \textbf{84.07}\stdvw{0.13} \\
    
    & \textit{Decay}    
    & \textbf{64.46}\stdvw{0.18}
    & 82.74\stdvw{0.25} \\
    \bottomrule
    \end{tabular}
    }
    \vspace{-1.0em}
\end{wraptable}
\noindent{\textbf{Choice of \texorpdfstring{$\zeta$}{zeta} Mode for \textbf{\texttt{$\boldsymbol{W}_{c,\mathrm{rm}}^{t}$}}.}}
Table~\ref{tab:6} compares two modes of the coefficient \(\zeta\) in \textbf{\texttt{$\boldsymbol{W}_{c,\mathrm{rm}}^{t}$}}: the \textit{constant mode} and the \textit{decay mode}. 
As shown in Table~\ref{tab:6}, the \textit{constant mode} performs better under the \texttt{IID} setting, while the \textit{decay mode} is more effective under the \texttt{Non-IID} setting. 
Under \texttt{IID}, clients follow relatively consistent optimization directions, so maintaining a stronger Round-Matching term helps preserve cross-round continuity and alleviate \ding{183}. 
In contrast, under \texttt{Non-IID}, client objectives are more heterogeneous, and a fixed large \(\zeta\) may over-preserve the previous local trajectory, making the initialization less adaptive to the current local objective. 
The \textit{decay mode} keeps the stabilizing effect of \textbf{\texttt{$\boldsymbol{W}_{c,\mathrm{rm}}^{t}$}} in early rounds while gradually reducing its influence, allowing \textbf{\texttt{$\boldsymbol{\hat{W}}_{c,\mathrm{ga}}^{t}$}} to better guide the initialization with current local optimization signals. 
Therefore, we adopt the \textit{constant mode} for \texttt{IID} settings and the \textit{decay mode} for \texttt{Non-IID} settings.

\begin{wraptable}{r}{0.50\textwidth}
    \vspace{-1.4em}
    \centering
    \caption{
    Analysis of client-agnostic and client-specific initialization variants in \textbf{\texttt{FedSmoothLoRA}} on CIFAR-100 with ViT-Small.
    FedAvgLoRA is reported as the baseline. The best results are highlighted in \textbf{bold}.
    }
    \label{tab:7}
    \resizebox{\linewidth}{!}{
    \begin{tabular}{l|c|cc}
    \toprule
    \textbf{Method} 
    & \textbf{Client-spec.}
    & \textbf{\texttt{Non-IID}} 
    & \textbf{\texttt{IID}} \\
    \midrule
    FedAvgLoRA
    & No
    & 57.36\stdvw{0.07}
    & 73.97\stdvw{0.11} \\
    \midrule

    \multicolumn{4}{l}{\textbf{\texttt{FedSmoothLoRA}}} \\

    \quad \(\hookrightarrow\) w/ weight-based SVD
    & No
    & 62.25\stdvw{0.23}
    & 83.21\stdvw{0.14} \\

    \quad \(\hookrightarrow\) w/ shared-gradient SVD\textsuperscript{*}
    & No
    & 63.86\stdvw{0.15}
    & 83.98\stdvw{0.06} \\

    {\cellcolor{mypink}\quad \(\hookrightarrow\) w/ client-specific $\boldsymbol{\hat{W}}_{c,\mathrm{ga}}^{t}$}
    & {\cellcolor{mypink}Yes}
    & {\cellcolor{mypink}\textbf{64.46}\stdvw{0.18}}
    & {\cellcolor{mypink}\textbf{84.07}\stdvw{0.13}} \\
    \bottomrule
    \end{tabular}
    }

    \vspace{0.2em}
    \parbox{\linewidth}{\footnotesize
    \textsuperscript{*}Idealized diagnostic baseline requiring shared client gradients; not intended for practical federated deployment.
    }
    \vspace{-1.0em}
\end{wraptable}

\noindent{\textbf{Effect of Client-Specific $\boldsymbol{\hat{W}}_{c,\mathrm{ga}}^{t}$.}}
Table~\ref{tab:7} compares client-agnostic and client-specific initialization variants for \textbf{\texttt{FedSmoothLoRA}}.
The weight-based SVD variant follows the initialization strategy used in FRLoRA, which derives the LoRA starting point from the shared model-weight structure and is therefore client-agnostic.
We further include shared-gradient SVD as an idealized diagnostic baseline to examine the effect of gradient-based initialization in a client-agnostic form.
As shown in Table~\ref{tab:7}, shared-gradient SVD consistently outperforms weight-based SVD under both \texttt{Non-IID} and \texttt{IID} settings.
This is consistent with observations in centralized LoRA-GA~\cite{wang2024loraga}, indicating that gradient-based SVD provides a more effective initialization than weight-based SVD.
Moreover, the client-specific \(\boldsymbol{\hat{W}}_{c,\mathrm{ga}}^{t}\) further improves over shared-gradient SVD, especially under the \texttt{Non-IID} setting.
This suggests that, in federated learning, gradient-based SVD not only inherits the advantage observed in centralized settings, namely providing a more optimization-aware initialization than weight-based SVD, but also brings an additional benefit when adapted to each client's local data.
By constructing the gradient-aligned initialization from each client's own local gradient signal, FedSmoothLoRA provides a more personalized starting point for heterogeneous local objectives and achieves better performance.

\section{Conclusion}

This paper proposes \textbf{FedSmoothLoRA}, a federated LoRA tuning framework that improves local training stability and client-aware adaptation across communication rounds. 
FedSmoothLoRA introduces two complementary client-side initialization designs. 
The Round-Matching matrix $\boldsymbol{W}_{c,\mathrm{rm}}^{t}$ aligns each new round with the previous local training state, thereby reducing inter-round state mismatch and suppressing abrupt parameter shifts.
The Gradient-Aligned matrix $\boldsymbol{\hat{W}}_{c,\mathrm{ga}}^{t}$ constructs the local LoRA starting state from full-model gradients estimated on a small local mini-batch, thereby incorporating client-specific optimization signals into the initialization process.
Together, $\boldsymbol{W}_{c,\mathrm{rm}}^{t}$ and $\boldsymbol{\hat{W}}_{c,\mathrm{ga}}^{t}$ enable smoother training dynamics, faster convergence, and stronger final performance. 
Extensive experiments on vision and natural language generation tasks validate the effectiveness of FedSmoothLoRA under both IID and Non-IID federated settings.

\bibliographystyle{unsrt}
\bibliography{neurips_2026}
\normalsize
\clearpage



\appendix

This appendix provides additional details and analyses for FedSmoothLoRA.
It includes:
\begin{itemize}
    \item Detailed experimental settings, provided in Section~\ref{app:exp_set}.
    \item Additional analysis of time, memory, and communication costs, summarized in Section~\ref{app:cost}.
    \item Partial participation experiments, presented in Section~\ref{app:part}.
    \item Analysis of inter-round state consistency and the role of $\boldsymbol{W}_{c,\mathrm{rm}}^{t}$, provided in Section~\ref{app:inter_round_consistency}.
    \item Explanation of server aggregation in federated LoRA, discussed in Section~\ref{app:aggregation}.
    \item Scalability analysis, shown in Section~\ref{app:scalability}.
    \item Limitations, discussed in Section~\ref{app:limit}.
    \item Compute resources, summarized in Section~\ref{app:resource}.
    \item Broader impacts, discussed in Section~\ref{app:broader}.
\end{itemize}

\section{Details of Experimental Settings}
\label{app:exp_set}

\subsection{Image Classification}

For the image classification experiments, we evaluate all methods on CIFAR-100 with 5 clients under both \texttt{IID} and label-skewed \texttt{Non-IID} partitions.
The 50,000 training samples are evenly assigned to clients in the \texttt{IID} setting.
For the \texttt{Non-IID} setting, we generate client label distributions using a Dirichlet distribution with concentration parameter \(\beta=0.1\).
Each client further splits its local data into 80\% for training and 20\% for validation.
All experiments are repeated over three runs, and we report the average performance together with the standard deviation.

We use ViT-Small pretrained on ImageNet-21k as the backbone.
LoRA is applied to the attention qkv layers and the classification head.
The LoRA rank is set to \(r=2\), and the scaling factor is set to \(\alpha=4\).
Each client trains for one local epoch per communication round.
We use the SGD optimizer with a cosine learning-rate schedule and an initial learning rate of 0.01.
The local training batch size is 64.
For FedSmoothLoRA, the stabilization hyperparameter in \(\boldsymbol{\hat{W}}_{c,\mathrm{ga}}^{t}\) is set to \(\gamma=256\), and the calibration mini-batch used to estimate the local full-model gradient contains 8 samples.
The coefficient \(\zeta\) is set to 1 under the \texttt{IID} setting and follows a cosine decay schedule from 1 to 0.6 under the \texttt{Non-IID} setting.

\subsection{Natural Language Generation}

For natural language generation, we evaluate FedSmoothLoRA on math reasoning, code generation, and multilingual chat tasks using LLaMA-3.2-1B as the backbone.
Unless otherwise specified, all experiments are trained for 10 communication rounds with 200 local training steps per round.
LoRA is applied to all linear layers, with rank \(r=32\) and scaling factor \(\alpha=64\).
All experiments are repeated over three runs, and we report the average performance together with the standard deviation.

For the math task, the model is trained on a 100k subset of MetaMathQA and evaluated on GSM8K using accuracy.
For the code task, the model is trained on a 100k subset of Code-Feedback and evaluated on HumanEval using Pass@1.
For the multilingual chat task, the model is trained on the Aya dataset under a language-skewed \texttt{Non-IID} setting, where each client corresponds to a different language distribution.
We evaluate the chat model on the Aya open-ended multilingual benchmark provided by FedLLM-Bench.
The evaluation set contains 140 instruction-reference pairs across 7 languages, including English, Standard Arabic, Russian, Simplified Chinese, Portuguese, French, and Spanish, with 20 samples per language.
The evaluation protocol consists of three stages: 
(1) generating model responses for all test instructions; 
(2) using an LLM-as-a-Judge scorer, with google/gemini-2.5-flash-lite as the default judge, to assign scores from 1 to 10 based on language consistency, instruction fulfillment, semantic and factual comprehensibility, and grammatical fluency; and 
(3) aggregating per-language mean scores and the overall mean score as the final performance metrics.

We use the AdamW optimizer with \(\beta_1=0.9\) and \(\beta_2=0.999\).
The learning rate follows a cosine schedule from \(2\times10^{-5}\) to \(10^{-6}\).
The batch size is 32 with gradient accumulation of 4.
The maximum sequence length is set to 512 for math and 1024 for code and chat tasks.
For FedSmoothLoRA, the stabilization hyperparameter in \(\boldsymbol{\hat{W}}_{c,\mathrm{ga}}^{t}\) is set to \(\gamma=256\), and the calibration mini-batch size used for gradient estimation is set to 128 unless otherwise specified.
\section{Time and Memory Cost}
\label{app:cost}

FedSmoothLoRA introduces additional computation mainly from two operations:
(1) estimating a local full-model gradient on a small calibration mini-batch for $\boldsymbol{\hat{W}}_{c,\mathrm{ga}}^{t}$, and
(2) performing low-rank SVD approximation for client-side initialization and server-side aggregation.
In practice, these operations introduce only a small overhead compared with local fine-tuning.

To reduce the cost of SVD on large matrices, we use randomized low-rank SVD following~\cite{halko2011finding}, implemented by \texttt{torch.svdlowrank}.
In our experiments, we set the number of iterations to 8 and the oversampling parameter to \(4r\).
This provides a practical trade-off between approximation quality and computational cost.

Table~\ref{tab:cost_time_memory} reports the time and memory cost on the math task.
FedSmoothLoRA introduces extra computation for gradient-aligned initialization and client-side SVD operations.
However, compared with the total local training time, these costs are relatively small.
Moreover, the peak memory cost remains comparable to FedAvgLoRA in our implementation.

\begin{table}[htbp]
    \centering
    \caption{Time and memory cost on the math task, evaluated on an A6000 GPU and 64 Intel(R) Xeon(R) Gold 6226R CPUs @ 2.90GHz.}
    \begin{tabular}{l|c|ccc|c}
    \toprule
    \multirow{2}{*}{Method}
    & \multirow{2}{*}{\begin{tabular}{c}Server\\Fast SVD\end{tabular}}
    & \multicolumn{3}{c|}{Time Cost of Each Client}
    & \multirow{2}{*}{\begin{tabular}{c}Memory Cost\\of Each Client\end{tabular}} \\
    \cline{3-5}
    & & Fast SVD & GA Initialization & Training & \\
    \midrule
    FedAvgLoRA
    & -- & -- & -- & \textasciitilde 7min & \textasciitilde 34GB \\
    FedSmoothLoRA
    & 2s & 2s $\times$ 2 & 18s & \textasciitilde 7min & \textasciitilde 34GB \\
    \bottomrule
    \end{tabular}
    \label{tab:cost_time_memory}
\end{table}

We further study the influence of the number of SVD iterations.
As shown in Table~\ref{tab:svd_iterations}, increasing the number of SVD iterations slightly improves the final validation loss, but also introduces additional training time.
Since the performance gain becomes marginal when using more iterations, we set the number of SVD iterations to 8 in our experiments.

\begin{table}[htbp]
    \centering
    \caption{Effect of the number of randomized SVD iterations in FedSmoothLoRA on the math task.}
    \begin{tabular}{l|cccc}
        \toprule
        Number of SVD Iterations & 8 & 16 & 32 & 64 \\
        \midrule
        Total Training Time
        & \textasciitilde 3h39min
        & \textasciitilde 3h41min
        & \textasciitilde 3h43min
        & \textasciitilde 3h45min \\
        Final Validation Loss
        & 0.2458
        & 0.2457
        & 0.2454
        & 0.2453 \\
        \bottomrule
    \end{tabular}
    \label{tab:svd_iterations}
\end{table}
\section{FedSmoothLoRA in Partial Participation Setting}
\label{app:part}

FedSmoothLoRA can also be extended to the partial device participation setting, where only a subset of clients participate in each communication round.
Different from the full participation setting, a client may remain inactive for several rounds before rejoining training.
Therefore, when an inactive client returns, the Round-Matching matrix should account for the accumulated server-side updates during its inactive period.

Let \(\hat{t}_c < t\) denote the most recent round before round \(t\) in which client \(c\) participated.
When client \(c\) rejoins training at round \(t\), the Round-Matching matrix is defined as
\begin{flalign}
\boldsymbol{W}_{c,\mathrm{rm}}^{t}
\leftarrow
\boldsymbol{B}_{c}^{\hat{t}_c}\boldsymbol{A}_{c}^{\hat{t}_c}
-
\sum_{\tau=\hat{t}_c+1}^{t}
\boldsymbol{B}_{s}^{\tau}\boldsymbol{A}_{s}^{\tau},
\label{eq:partial_rm}
\end{flalign}
where \(\hat{t}_c\) represents the last active round of client \(c\).
This formulation aligns the retained client-side LoRA state with the accumulated server-side LoRA updates during the inactive period, thereby reducing state mismatch when the client rejoins training.

We further evaluate FedSmoothLoRA under partial device participation on CIFAR-100.
The experimental setup follows Section~\ref{sec:4.1}, except that 3 out of 5 clients are randomly selected for local training in each communication round.
As shown in Table~\ref{tab:partial_participation}, FedSmoothLoRA continues to outperform FedAvgLoRA under partial participation, indicating that the proposed initialization strategy remains effective when only a subset of clients is active in each round.

\begin{table}[htbp]
    \centering
    \caption{
    Experiments under partial device participation on CIFAR-100 with the ViT-Small model.
    In each communication round, 3 out of 5 clients are randomly selected for local training.
    The best result excluding FedAvg (Full Finetune) is highlighted in \textbf{bold}.
    }
    \label{tab:partial_participation}
    \begin{tabular}{lcccc}
        \toprule
        \textbf{Method}
        & \textbf{Active}
        & \textbf{Inactive}
        & \textbf{Total}
        & \textbf{Accuracy} \\
        \midrule
        FedAvg (Full Finetune)
        & \(83\mathrm{MB} \times 2 \times 3\)
        & --
        & \(83\mathrm{MB} \times 6\)
        & 90.20 \\
        
        FedAvgLoRA
        & \(157\mathrm{KB} \times 2 \times 3\)
        & --
        & \(157\mathrm{KB} \times 6\)
        & 73.74 \\

        \textbf{\texttt{FedSmoothLoRA}}
        & \(157\mathrm{KB} \times 2 \times 3\)
        & \(157\mathrm{KB} \times 1 \times 2\)
        & \(157\mathrm{KB} \times 8\)
        & \textbf{82.60} \\
        \bottomrule
    \end{tabular}
\end{table}

\section{Mechanistic Analysis of Inter-Round State Continuity}
\label{app:inter_round_consistency}

In this section, we provide a mechanistic analysis of how the Round-Matching matrix 
\(\boldsymbol{W}_{c,\mathrm{rm}}^{t}\) improves inter-round state continuity in FedSmoothLoRA.
Rather than aiming to establish a full convergence guarantee, this analysis characterizes the discrepancy between the effective starting state of a client at the current round and its effective local endpoint from the previous round.
For clarity, we focus on the full-participation setting and consider a client \(c\) at round \(t>0\).
The same analysis can be extended to partial participation by replacing round \(t-1\) with the most recent round in which client \(c\) participated.

To make the effect of \(\boldsymbol{W}_{c,\mathrm{rm}}^{t}\) explicit, we compare two effective model states that determine the discontinuity observed across communication rounds:
\begin{itemize}
    \item \(\boldsymbol{S}_c^t\): the effective starting state from which client \(c\) begins local optimization at round \(t\);
    \item \(\boldsymbol{E}_c^{t-1}\): the effective endpoint reached by client \(c\) after local optimization at round \(t-1\).
\end{itemize}

\noindent\textbf{Definition of the starting state.}
At round \(t\), FedSmoothLoRA performs local training with the shifted backbone 
\(\boldsymbol{W}_c^t-s\hat{\boldsymbol{W}}_{c,\mathrm{ga}}^t\) 
and the initialized LoRA factors 
\((\boldsymbol{B}_{c,\mathrm{init}}^t,\boldsymbol{A}_{c,\mathrm{init}}^t)\). 
Therefore, the effective starting state is defined as
\begin{flalign}
\boldsymbol{S}_c^t
:=
\boldsymbol{W}_c^t
-
s\hat{\boldsymbol{W}}_{c,\mathrm{ga}}^t
+
s\boldsymbol{B}_{c,\mathrm{init}}^t
\boldsymbol{A}_{c,\mathrm{init}}^t .
\label{eq:app_start_state_def}
\end{flalign}

\noindent\textbf{Definition of the endpoint state.}
At round \(t-1\), local training is performed with the shifted backbone 
\(\boldsymbol{W}_c^{t-1}-s\hat{\boldsymbol{W}}_{c,\mathrm{ga}}^{t-1}\). 
Let 
\((\widetilde{\boldsymbol{B}}_c^{t-1},\widetilde{\boldsymbol{A}}_c^{t-1})\) 
denote the temporary LoRA factors obtained immediately after local optimization but before the final SVD projection. 
The effective local endpoint is therefore defined as
\begin{flalign}
\boldsymbol{E}_c^{t-1}
:=
\boldsymbol{W}_c^{t-1}
-
s\hat{\boldsymbol{W}}_{c,\mathrm{ga}}^{t-1}
+
s\widetilde{\boldsymbol{B}}_c^{t-1}
\widetilde{\boldsymbol{A}}_c^{t-1}.
\label{eq:app_end_state_def}
\end{flalign}

The following proposition decomposes the discrepancy between these two states and makes the role of the Round-Matching matrix explicit.

\begin{proposition}[Inter-round state-discrepancy decomposition]
\label{prop:inter_round_consistency}
Let \(\boldsymbol{S}_c^t\) and \(\boldsymbol{E}_c^{t-1}\) be defined by Eq.~\eqref{eq:app_start_state_def} and Eq.~\eqref{eq:app_end_state_def}.
In FedSmoothLoRA, \(s=\alpha/\sqrt{r}>0\) and the coefficient \(\zeta\) is chosen from \([0,1]\) in both the constant and decay modes. Therefore, \(s(1-\zeta)\ge 0\).
Then their discrepancy satisfies
\begin{flalign}
\boldsymbol{S}_c^t
-
\boldsymbol{E}_c^{t-1}
=
s(1-\zeta)
\left(
\boldsymbol{B}_s^t\boldsymbol{A}_s^t
-
\boldsymbol{B}_c^{t-1}\boldsymbol{A}_c^{t-1}
\right)
+
s\boldsymbol{\epsilon}_{c,\mathrm{init}}^t
+
s\boldsymbol{\epsilon}_{c,\mathrm{end}}^{t-1},
\label{eq:app_state_gap}
\end{flalign}
where \(\boldsymbol{\epsilon}_{c,\mathrm{init}}^t\) is the rank-\(r\) SVD approximation error in the initialization step at round \(t\), and \(\boldsymbol{\epsilon}_{c,\mathrm{end}}^{t-1}\) is the rank-\(r\) SVD approximation error in the final client-side projection at the end of round \(t-1\). 
Consequently,
\begin{flalign}
\left\|
\boldsymbol{S}_c^t
-
\boldsymbol{E}_c^{t-1}
\right\|_F
\leq
s(1-\zeta)
\left\|
\boldsymbol{B}_s^t\boldsymbol{A}_s^t
-
\boldsymbol{B}_c^{t-1}\boldsymbol{A}_c^{t-1}
\right\|_F
+
s\left\|
\boldsymbol{\epsilon}_{c,\mathrm{init}}^t
\right\|_F
+
s\left\|
\boldsymbol{\epsilon}_{c,\mathrm{end}}^{t-1}
\right\|_F .
\label{eq:app_state_gap_bound}
\end{flalign}
\end{proposition}

\noindent\textbf{Proof.}
We prove the proposition by separately deriving the starting state at round \(t\), the endpoint state at round \(t-1\), and their difference.

\noindent\emph{Step 1: Deriving the starting state at round \(t\).}
At the beginning of round \(t\), client \(c\) merges the current server-side LoRA update into its local backbone:
\begin{flalign}
\boldsymbol{W}_c^t
=
\boldsymbol{W}_c^{t-1}
+
s\boldsymbol{B}_s^t\boldsymbol{A}_s^t .
\label{eq:app_client_backbone}
\end{flalign}
The Round-Matching matrix is defined as
\begin{flalign}
\boldsymbol{W}_{c,\mathrm{rm}}^t
=
\boldsymbol{B}_c^{t-1}\boldsymbol{A}_c^{t-1}
-
\boldsymbol{B}_s^t\boldsymbol{A}_s^t .
\label{eq:app_rm_def}
\end{flalign}
The final initialization matrix is
\begin{flalign}
\boldsymbol{W}_{c,\mathrm{init}}^t
=
\hat{\boldsymbol{W}}_{c,\mathrm{ga}}^t
+
\zeta
\boldsymbol{W}_{c,\mathrm{rm}}^t .
\label{eq:app_init_matrix}
\end{flalign}
The initialized LoRA factors are obtained by a rank-\(r\) SVD approximation:
\begin{flalign}
(\boldsymbol{B}_{c,\mathrm{init}}^t,\boldsymbol{A}_{c,\mathrm{init}}^t)
=
\operatorname{SVDApprox}
\left(
\boldsymbol{W}_{c,\mathrm{init}}^t;
r
\right).
\label{eq:app_init_svd}
\end{flalign}
We define the initialization SVD approximation error as
\begin{flalign}
\boldsymbol{\epsilon}_{c,\mathrm{init}}^t
:=
\boldsymbol{B}_{c,\mathrm{init}}^t
\boldsymbol{A}_{c,\mathrm{init}}^t
-
\boldsymbol{W}_{c,\mathrm{init}}^t .
\label{eq:app_init_error}
\end{flalign}
Equivalently,
\begin{flalign}
\boldsymbol{B}_{c,\mathrm{init}}^t
\boldsymbol{A}_{c,\mathrm{init}}^t
=
\hat{\boldsymbol{W}}_{c,\mathrm{ga}}^t
+
\zeta
\boldsymbol{W}_{c,\mathrm{rm}}^t
+
\boldsymbol{\epsilon}_{c,\mathrm{init}}^t .
\label{eq:app_init_product}
\end{flalign}
Substituting Eq.~\eqref{eq:app_init_product} into Eq.~\eqref{eq:app_start_state_def}, we obtain
\begin{flalign}
\boldsymbol{S}_c^t
&=
\boldsymbol{W}_c^t
-
s\hat{\boldsymbol{W}}_{c,\mathrm{ga}}^t
+
s
\left(
\hat{\boldsymbol{W}}_{c,\mathrm{ga}}^t
+
\zeta
\boldsymbol{W}_{c,\mathrm{rm}}^t
+
\boldsymbol{\epsilon}_{c,\mathrm{init}}^t
\right)
\nonumber\\
&=
\boldsymbol{W}_c^t
+
s\zeta
\boldsymbol{W}_{c,\mathrm{rm}}^t
+
s\boldsymbol{\epsilon}_{c,\mathrm{init}}^t .
\label{eq:app_start_simplified}
\end{flalign}
Then, substituting Eq.~\eqref{eq:app_client_backbone} and Eq.~\eqref{eq:app_rm_def} into Eq.~\eqref{eq:app_start_simplified}, we get
\begin{flalign}
\boldsymbol{S}_c^t
&=
\boldsymbol{W}_c^{t-1}
+
s\boldsymbol{B}_s^t\boldsymbol{A}_s^t
+
s\zeta
\left(
\boldsymbol{B}_c^{t-1}\boldsymbol{A}_c^{t-1}
-
\boldsymbol{B}_s^t\boldsymbol{A}_s^t
\right)
+
s\boldsymbol{\epsilon}_{c,\mathrm{init}}^t
\nonumber\\
&=
\boldsymbol{W}_c^{t-1}
+
s\zeta
\boldsymbol{B}_c^{t-1}\boldsymbol{A}_c^{t-1}
+
s(1-\zeta)
\boldsymbol{B}_s^t\boldsymbol{A}_s^t
+
s\boldsymbol{\epsilon}_{c,\mathrm{init}}^t .
\label{eq:app_start_expanded}
\end{flalign}

\noindent\emph{Step 2: Deriving the endpoint state at round \(t-1\).}
After local optimization at round \(t-1\), client \(c\) obtains temporary LoRA factors 
\((\widetilde{\boldsymbol{B}}_c^{t-1},\widetilde{\boldsymbol{A}}_c^{t-1})\). 
Before uploading and retaining the local update, FedSmoothLoRA performs the final SVD projection:
\begin{flalign}
(\boldsymbol{B}_c^{t-1},\boldsymbol{A}_c^{t-1})
=
\operatorname{SVDApprox}
\left(
\widetilde{\boldsymbol{B}}_c^{t-1}
\widetilde{\boldsymbol{A}}_c^{t-1}
-
\hat{\boldsymbol{W}}_{c,\mathrm{ga}}^{t-1};
r
\right).
\label{eq:app_end_svd}
\end{flalign}
Define the final SVD approximation error as
\begin{flalign}
\boldsymbol{\epsilon}_{c,\mathrm{end}}^{t-1}
:=
\boldsymbol{B}_c^{t-1}\boldsymbol{A}_c^{t-1}
-
\left(
\widetilde{\boldsymbol{B}}_c^{t-1}
\widetilde{\boldsymbol{A}}_c^{t-1}
-
\hat{\boldsymbol{W}}_{c,\mathrm{ga}}^{t-1}
\right).
\label{eq:app_end_error}
\end{flalign}
Equivalently,
\begin{flalign}
\widetilde{\boldsymbol{B}}_c^{t-1}
\widetilde{\boldsymbol{A}}_c^{t-1}
-
\hat{\boldsymbol{W}}_{c,\mathrm{ga}}^{t-1}
=
\boldsymbol{B}_c^{t-1}\boldsymbol{A}_c^{t-1}
-
\boldsymbol{\epsilon}_{c,\mathrm{end}}^{t-1}.
\label{eq:app_end_relation}
\end{flalign}
Substituting Eq.~\eqref{eq:app_end_relation} into the endpoint definition in Eq.~\eqref{eq:app_end_state_def}, we obtain
\begin{flalign}
\boldsymbol{E}_c^{t-1}
&=
\boldsymbol{W}_c^{t-1}
+
s
\left(
\widetilde{\boldsymbol{B}}_c^{t-1}
\widetilde{\boldsymbol{A}}_c^{t-1}
-
\hat{\boldsymbol{W}}_{c,\mathrm{ga}}^{t-1}
\right)
\nonumber\\
&=
\boldsymbol{W}_c^{t-1}
+
s\boldsymbol{B}_c^{t-1}\boldsymbol{A}_c^{t-1}
-
s\boldsymbol{\epsilon}_{c,\mathrm{end}}^{t-1}.
\label{eq:app_end_rewrite}
\end{flalign}

\noindent\emph{Step 3: Deriving the discrepancy between the two states.}
Combining Eq.~\eqref{eq:app_start_expanded} and Eq.~\eqref{eq:app_end_rewrite}, we have
\begin{flalign}
\boldsymbol{S}_c^t
-
\boldsymbol{E}_c^{t-1}
&=
\left[
\boldsymbol{W}_c^{t-1}
+
s\zeta
\boldsymbol{B}_c^{t-1}\boldsymbol{A}_c^{t-1}
+
s(1-\zeta)
\boldsymbol{B}_s^t\boldsymbol{A}_s^t
+
s\boldsymbol{\epsilon}_{c,\mathrm{init}}^t
\right]
\nonumber\\
&\quad -
\left[
\boldsymbol{W}_c^{t-1}
+
s\boldsymbol{B}_c^{t-1}\boldsymbol{A}_c^{t-1}
-
s\boldsymbol{\epsilon}_{c,\mathrm{end}}^{t-1}
\right]
\nonumber\\
&=
s(1-\zeta)
\left(
\boldsymbol{B}_s^t\boldsymbol{A}_s^t
-
\boldsymbol{B}_c^{t-1}\boldsymbol{A}_c^{t-1}
\right)
+
s\boldsymbol{\epsilon}_{c,\mathrm{init}}^t
+
s\boldsymbol{\epsilon}_{c,\mathrm{end}}^{t-1}.
\label{eq:app_gap_derivation}
\end{flalign}
This proves Eq.~\eqref{eq:app_state_gap}. 
Applying the triangle inequality to Eq.~\eqref{eq:app_gap_derivation} gives Eq.~\eqref{eq:app_state_gap_bound}.

Eq.~\eqref{eq:app_state_gap_bound} provides a mechanistic decomposition of the inter-round state discrepancy.
The first term corresponds to the mismatch between the current server-side LoRA update and the previous client-side LoRA update, scaled by \(1-\zeta\).
The remaining two terms are caused by the rank-\(r\) SVD approximations used in the initialization step and in the final client-side projection, respectively.

The two approximation terms,
\(\boldsymbol{\epsilon}_{c,\mathrm{init}}^t\) and 
\(\boldsymbol{\epsilon}_{c,\mathrm{end}}^{t-1}\), reflect the information loss introduced when projecting the corresponding full matrices back to rank-\(r\) LoRA forms.
When these approximation errors are small, the dominant part of the inter-round discrepancy is controlled by
\begin{flalign}
s(1-\zeta)
\left(
\boldsymbol{B}_s^t\boldsymbol{A}_s^t
-
\boldsymbol{B}_c^{t-1}\boldsymbol{A}_c^{t-1}
\right).
\end{flalign}
Therefore, a larger \(\zeta\) more strongly preserves the previous client-side local state and reduces the mismatch between the next-round starting state and the previous-round local endpoint.

In the idealized case where the SVD approximation errors vanish, i.e.,
\(\boldsymbol{\epsilon}_{c,\mathrm{init}}^t=\boldsymbol{0}\) and
\(\boldsymbol{\epsilon}_{c,\mathrm{end}}^{t-1}=\boldsymbol{0}\), the discrepancy reduces to
\begin{flalign}
\boldsymbol{S}_c^t
-
\boldsymbol{E}_c^{t-1}
=
s(1-\zeta)
\left(
\boldsymbol{B}_s^t\boldsymbol{A}_s^t
-
\boldsymbol{B}_c^{t-1}\boldsymbol{A}_c^{t-1}
\right).
\end{flalign}
If we further set \(\zeta=1\), then
\begin{flalign}
\boldsymbol{S}_c^t
=
\boldsymbol{E}_c^{t-1}.
\end{flalign}
This identity does not by itself imply a convergence guarantee, but it clarifies the role of \(\boldsymbol{W}_{c,\mathrm{rm}}^{t}\): it reduces inter-round state mismatch by aligning the starting state of the next round with the effective local endpoint of the previous round.

This analysis also explains the choice of \(\zeta\) mode used in FedSmoothLoRA.
Under the \texttt{IID} setting, clients tend to follow relatively consistent optimization trajectories, so we use the constant mode with \(\zeta=1\) to fully preserve cross-round state continuity.
Under the \texttt{Non-IID} setting, however, different clients may follow more heterogeneous local objectives.
In this case, using a fixed large \(\zeta\) may over-constrain the current local update toward the previous local trajectory.
Therefore, we use a decay mode for \(\zeta\), which preserves the stabilizing effect of Round-Matching in early rounds while gradually reducing its influence, allowing the Gradient-Aligned matrix to better adapt the initialization to the current client-specific optimization signal.
This provides a mechanistic explanation for the smoother training dynamics observed in our experiments.

\section{Explanation of Server Aggregation in Federated LoRA}
\label{app:aggregation}

FedSmoothLoRA adopts \textit{Full-Rank Aggregation with SVD Approximation} on the server side.
This strategy first averages the client updates in the full matrix space and then projects the aggregated update back to a rank-\(r\) LoRA form.
Specifically, the server first aggregates the effective update matrices from different clients:
\begin{flalign}
\Delta \boldsymbol{W}_s^{t+1}
=
\frac{1}{N}
\sum_{c\in C}
n_c
\boldsymbol{B}_c^t\boldsymbol{A}_c^t,
\end{flalign}
and then applies SVD approximation to obtain the server-side LoRA factors:
\begin{flalign}
(\boldsymbol{B}_s^{t+1}, \boldsymbol{A}_s^{t+1})
=
\operatorname{SVDApprox}(\Delta \boldsymbol{W}_s^{t+1}; r).
\end{flalign}

This design is important because LoRA uses a two-factor parameterization, where the effective update is represented as \(\boldsymbol{B}\boldsymbol{A}\).
Directly averaging the factors \(\boldsymbol{B}\) and \(\boldsymbol{A}\) across clients may be unsuitable when clients start from different client-specific initializations.
In FedSmoothLoRA, the Round-Matching matrix and the Gradient-Aligned matrix make the local LoRA factors client-specific.
As a result, different clients may learn LoRA factors that represent useful effective updates but are not necessarily aligned at the factor level.
Directly averaging these factors can therefore produce a poor aggregated adapter.
By contrast, Full-Rank Aggregation averages the effective update matrices before the SVD projection, producing a more consistent global update.

Table~\ref{tab:aggregation_round1} shows the accuracy before and after aggregation at the end of the first communication round.
The server model obtained by Full-Rank Aggregation achieves higher accuracy than that obtained by direct factor averaging.
Table~\ref{tab:aggregation_final} further reports the final CIFAR-100 results, where FedSmoothLoRA with Full-Rank Aggregation substantially outperforms the direct factor averaging variant.
These results confirm that the server aggregation strategy plays an important role in federated LoRA tuning.

\begin{table}[htbp]
    \centering
    \caption{
    Top-1 test accuracy before and after server aggregation at the end of round 1.
    The experiments are conducted on CIFAR-100 with the ViT-Small model. Results are from separate runs with different random seeds.
    }
    \resizebox{\linewidth}{!}{%
    \begin{tabular}{l|c|ccccc|c}
        \toprule
        \multirow{2}{*}{Method}
        & \multirow{2}{*}{Server Aggregation}
        & \multicolumn{5}{c|}{Before Aggregation}
        & \multicolumn{1}{c}{After Aggregation} \\
        \cmidrule(lr){3-7} \cmidrule(lr){8-8}
        &
        & Client 0
        & Client 1
        & Client 2
        & Client 3
        & Client 4
        & Server \\
        \midrule
        FedSmoothLoRA
        & Direct Factor Averaging
        & 26.66
        & 33.71
        & 30.05
        & 30.99
        & 30.96
        & 17.58 \\

        FedSmoothLoRA
        & Full-Rank Aggregation
        & 26.79
        & 34.29
        & 30.33
        & 31.46
        & 31.02
        & 27.49 \\
        \bottomrule
    \end{tabular}
    }
    \label{tab:aggregation_round1}
\end{table}

\begin{table}[htbp]
    \centering
    \caption{
    Comparison of different server aggregation strategies on CIFAR-100 with the ViT-Small model.
    The best results are highlighted in \textbf{bold}.
    }
    \begin{tabular}{l|ccc}
    \toprule
    Method
    & Server Aggregation
    & CIFAR-100 IID
    & CIFAR-100 Non-IID \\
    \midrule
    FedAvgLoRA
    & Direct Factor Averaging
    & 73.97{\tiny $\pm$ 0.11}
    & 57.36{\tiny $\pm$ 0.07} \\

    FedSmoothLoRA
    & Direct Factor Averaging
    & 74.92{\tiny $\pm$ 0.47}
    & 56.03{\tiny $\pm$ 0.21} \\

    FedSmoothLoRA
    & Full-Rank Aggregation
    & \textbf{84.07{\tiny $\pm$ 0.13}}
    & \textbf{64.46{\tiny $\pm$ 0.18}} \\
    \bottomrule
    \end{tabular}
    \label{tab:aggregation_final}
\end{table}
\section{Scalability Analysis}
\label{app:scalability}
In this section, we conduct a comprehensive scalability analysis of FedAvgLoRA, FRLoRA, and FedSmoothLoRA along four dimensions: model scale, LoRA rank, local steps, and the number of participating clients. Unless otherwise specified, all experiments are conducted on the math task using LLaMA-3.2-1B with LoRA rank $r{=}32$, 200 local steps, and 3 clients. The results are summarized in Figure~\ref{fig:scalability}. First, we extend the experiments to larger models, including LLaMA-3.2-3B, LLaMA-2-7B, and LLaMA-2-13B. As shown in Figure~\ref{fig:scalability}(a), FRLoRA and FedSmoothLoRA consistently outperform FedAvgLoRA across different model scales, demonstrating their practicality and scalability when applied to larger backbone models. Second, we vary the LoRA rank $r \in \{32, 64, 128\}$ to study the influence of adaptation capacity. As shown in Figure~\ref{fig:scalability}(b), increasing the rank generally improves performance, but it also introduces higher computational and communication costs. Under the same rank setting, both FRLoRA and FedSmoothLoRA achieve better performance than FedAvgLoRA, indicating that their improvements are not merely caused by using a larger parameter budget. Notably, FedSmoothLoRA with $r{=}32$ already outperforms FedAvgLoRA with $r{=}128$, further validating the effectiveness of the proposed optimization design under a smaller resource budget. Third, we evaluate different communication budgets by varying the number of communication steps in $\{100, 200, 400\}$. As shown in Figure~\ref{fig:scalability}(c), FRLoRA and FedSmoothLoRA consistently maintain clear advantages over FedAvgLoRA across all communication settings, suggesting that the proposed methods remain effective under both limited and relatively sufficient communication budgets. Finally, we increase the number of participating clients to $\{3, 6, 12\}$. As shown in Figure~\ref{fig:scalability}(d), FRLoRA and FedSmoothLoRA continue to outperform FedAvgLoRA as the federation scale increases, confirming their robustness in more realistic federated scenarios with greater client participation. Overall, these results show that the proposed methods can scale effectively across model size, adaptation rank, communication budget, and federation size, while FedSmoothLoRA achieves the most consistent performance gains among the compared methods.

\begin{figure*}[t]
    \centering
    \includegraphics[width=\textwidth]{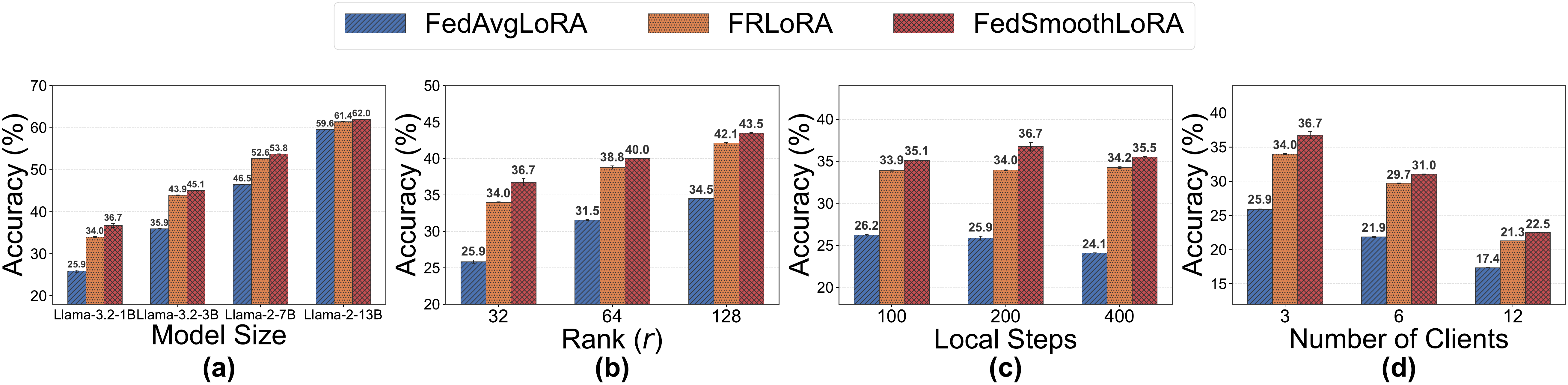}
    \caption{Scalability analysis of FedAvgLoRA, FRLoRA, and FedSmoothLoRA on the math task. We evaluate performance across four dimensions: (a)~model scale (LLaMA-3.2-1B, LLaMA-3.2-3B, LLaMA-2-7B, and LLaMA-2-13B), (b)~LoRA rank ($r \in \{32, 64, 128\}$), (c)~local steps ($\{100, 200, 400\}$), and (d)~number of participating clients ($\{3, 6, 12\}$). Unless otherwise specified, experiments use LLaMA-3.2-1B with $r{=}32$, 200 local steps, and 3 clients. The best results are highlighted in \textbf{bold}.}
    \label{fig:scalability}
\end{figure*}

\section{Limitations}
\label{app:limit}

Despite the promising results of \textbf{\texttt{FedSmoothLoRA}}, two key limitations remain:
(1) The current evaluations are limited to CIFAR-100, GSM8K, HumanEval, and Aya multilingual chat tasks. 
Although these benchmarks cover both image classification and natural language generation scenarios, further evaluations on more diverse benchmarks, task types, and real-world federated settings are necessary to comprehensively assess the generalization and robustness of FedSmoothLoRA.
(2) The current evaluations mainly focus on LLaMA-family models and do not include other representative model families, such as Qwen, Mistral, or Gemma. 
Future work will extend the evaluation to more diverse backbone model families to better understand the generality and robustness of FedSmoothLoRA across different model architectures.

\section{Compute Resources}
\label{app:resource}

All experiments are conducted using different hardware configurations depending on the task. 
For image classification experiments, we use a single NVIDIA RTX 3060 GPU with 12GB memory and an 8-core Intel Xeon E3-1231 v3 CPU @ 3.40GHz. 
For natural language generation experiments based on LLaMA-3.2-1B, we use a single NVIDIA A6000 GPU with 48GB memory and a 64-core Intel Xeon Gold 6226R CPU @ 2.90GHz. 
Experiments involving larger LLaMA-2-13B models are conducted on a single NVIDIA A800 GPU with 80GB memory and a 160-core Intel Xeon Platinum 8383C CPU @ 2.70GHz.

\section{Broader Impacts}
\label{app:broader}

We believe this work has the potential for a positive social impact. 
At present, adapting large-scale foundation models typically requires centralized access to substantial computational resources and large datasets. 
By integrating Federated Learning with LoRA, FedSmoothLoRA provides a promising alternative for settings where data are isolated and computational resources are limited. 
However, directly combining Federated Learning and LoRA still faces several challenges, including the limited update space, unstable training dynamics, and slow convergence.

FedSmoothLoRA addresses these issues by enlarging the effective update space, improving cross-round training continuity, and providing client-specific initialization for local adaptation. 
These improvements make federated LoRA tuning more practical in distributed and resource-constrained environments. 
By reducing unnecessary communication and improving convergence efficiency, FedSmoothLoRA may also help lower the cost of model adaptation and contribute to more efficient use of computational resources.

Furthermore, as foundation models continue to scale, it becomes increasingly difficult for individuals, small organizations, and institutions with limited resources to participate in model adaptation. 
FedSmoothLoRA enables more effective use of distributed data and computation while keeping raw data local, which may promote broader accessibility and support more decentralized development of large models.

Finally, FedSmoothLoRA does not eliminate all risks associated with federated model training or large model adaptation. 
Although federated learning reduces the need for direct data sharing, model updates may still leak sensitive information in some scenarios. 
In addition, models trained with FedSmoothLoRA may still produce hallucinated, biased, or misleading outputs. 
Therefore, practical deployment should be combined with appropriate privacy protection, safety evaluation, and output monitoring mechanisms.

\newpage
\section*{NeurIPS Paper Checklist}

\begin{enumerate}

\item {\bf Claims}
    \item[] Question: Do the main claims made in the abstract and introduction accurately reflect the paper's contributions and scope?
    \item[] Answer: \answerYes{} 
    \item[] Justification: The abstract and introduction clearly state the main limitations of FedAvgLoRA, including limited update space, inter-round state mismatch, and client-agnostic initialization, and present FedSmoothLoRA as a method designed to address these issues. The claimed contributions are supported by the proposed Round-Matching and Gradient-Aligned components and by empirical evaluations on vision and natural language generation tasks under both IID and Non-IID federated settings.

\item {\bf Limitations}
    \item[] Question: Does the paper discuss the limitations of the work performed by the authors?
    \item[] Answer: \answerYes{} 
    \item[] Justification: The paper discusses limitations in Appendix~\ref{app:limit}, including the limited range of evaluated benchmarks and backbone model families.
    \begin{itemize}
        \item The answer \answerNA{} means that the paper has no limitation while the answer \answerNo{} means that the paper has limitations, but those are not discussed in the paper. 
        \item The authors are encouraged to create a separate ``Limitations'' section in their paper.
        \item The paper should point out any strong assumptions and how robust the results are to violations of these assumptions (e.g., independence assumptions, noiseless settings, model well-specification, asymptotic approximations only holding locally). The authors should reflect on how these assumptions might be violated in practice and what the implications would be.
        \item The authors should reflect on the scope of the claims made, e.g., if the approach was only tested on a few datasets or with a few runs. In general, empirical results often depend on implicit assumptions, which should be articulated.
        \item The authors should reflect on the factors that influence the performance of the approach. For example, a facial recognition algorithm may perform poorly when image resolution is low or images are taken in low lighting. Or a speech-to-text system might not be used reliably to provide closed captions for online lectures because it fails to handle technical jargon.
        \item The authors should discuss the computational efficiency of the proposed algorithms and how they scale with dataset size.
        \item If applicable, the authors should discuss possible limitations of their approach to address problems of privacy and fairness.
        \item While the authors might fear that complete honesty about limitations might be used by reviewers as grounds for rejection, a worse outcome might be that reviewers discover limitations that aren't acknowledged in the paper. The authors should use their best judgment and recognize that individual actions in favor of transparency play an important role in developing norms that preserve the integrity of the community. Reviewers will be specifically instructed to not penalize honesty concerning limitations.
    \end{itemize}

\item {\bf Theory assumptions and proofs}
    \item[] Question: For each theoretical result, does the paper provide the full set of assumptions and a complete (and correct) proof?
    \item[] Answer: \answerYes{}
    \item[] Justification: The paper provides a formal proposition on inter-round state-discrepancy decomposition in Appendix~\ref{app:inter_round_consistency}, together with the definitions, assumptions, and proof. The analysis is explicitly presented as a mechanistic characterization rather than a full convergence guarantee.

    \item {\bf Experimental result reproducibility}
    \item[] Question: Does the paper fully disclose all the information needed to reproduce the main experimental results of the paper to the extent that it affects the main claims and/or conclusions of the paper (regardless of whether the code and data are provided or not)?
    \item[] Answer: \answerYes{}  
    \item[] Justification: The experimental hyperparameters can be found in Section~\ref{experiment}, specifically in the paragraphs beginning with “Experimental Setting.” Additional implementation details are provided in Appendix~\ref{app:exp_set}. We also include the code for our algorithm in the supplementary material to ensure reproducibility. 

\item {\bf Open access to data and code}
    \item[] Question: Does the paper provide open access to the data and code, with sufficient instructions to faithfully reproduce the main experimental results, as described in supplemental material?
    \item[] Answer: \answerYes{}
    \item[] Justification:  We include the code for our algorithm in the supplementary material to ensure reproducibility. All datasets used in our experiments are open-source and have been properly declared and cited in Section~\ref{experiment}.


\item {\bf Experimental setting/details}
    \item[] Question: Does the paper specify all the training and test details (e.g., data splits, hyperparameters, how they were chosen, type of optimizer) necessary to understand the results?
    \item[] Answer: \answerYes{} 
    \item[] Justification: The paper specifies the datasets, client partitions, model backbones, LoRA ranks, scaling factors, optimizers, learning-rate schedules, batch sizes, local training steps, and evaluation protocols in Section~\ref{experiment} and Appendix~\ref{app:exp_set}.

\item {\bf Experiment statistical significance}
    \item[] Question: Does the paper report error bars suitably and correctly defined or other appropriate information about the statistical significance of the experiments?
    \item[] Answer: \answerYes{} 
    \item[] Justification: We report standard deviations together with the main evaluation results in Tables~\ref{tab:cifar100_results}, \ref{tab:math_code_results}, \ref{tab:chat_results}, \ref{tab:3}, \ref{tab:6}, and \ref{tab:7}. These standard deviations summarize variability across repeated runs under the same experimental settings.

\item {\bf Experiments compute resources}
    \item[] Question: For each experiment, does the paper provide sufficient information on the computer resources (type of compute workers, memory, time of execution) needed to reproduce the experiments?
    \item[] Answer: \answerYes{} 
    \item[] Justification: We discuss the computational resources used in our experiments in Appendix~\ref{app:resource}.
 
    
\item {\bf Code of ethics}
    \item[] Question: Does the research conducted in the paper conform, in every respect, with the NeurIPS Code of Ethics \url{https://neurips.cc/public/EthicsGuidelines}?
    \item[] Answer: \answerYes{} 
    \item[] Justification: This paper conforms to the NeurIPS Code of Ethics.


\item {\bf Broader impacts}
    \item[] Question: Does the paper discuss both potential positive societal impacts and negative societal impacts of the work performed?
    \item[] Answer: \answerYes{} 
    \item[] Justification: We have discussed this in Appendix~\ref{app:broader}.    
    
\item {\bf Safeguards}
    \item[] Question: Does the paper describe safeguards that have been put in place for responsible release of data or models that have a high risk for misuse (e.g., pre-trained language models, image generators, or scraped datasets)?
    \item[] Answer: \answerNA{} 
    \item[] Justification: The paper does not introduce or release a new high-risk dataset or pre-trained generative model. It studies a federated LoRA tuning algorithm using existing datasets and backbone models, so this question is not directly applicable.

\item {\bf Licenses for existing assets}
    \item[] Question: Are the creators or original owners of assets (e.g., code, data, models), used in the paper, properly credited and are the license and terms of use explicitly mentioned and properly respected?
    \item[] Answer: \answerYes{} 
    \item[] Justification: All models and datasets used in our paper are properly cited.

\item {\bf New assets}
    \item[] Question: Are new assets introduced in the paper well documented and is the documentation provided alongside the assets?
    \item[] Answer: \answerYes{} 
    \item[] Justification: Instructions for running our code, including environment setup, usage, and hyperparameter settings, are provided in the README.md file.

\item {\bf Crowdsourcing and research with human subjects}
    \item[] Question: For crowdsourcing experiments and research with human subjects, does the paper include the full text of instructions given to participants and screenshots, if applicable, as well as details about compensation (if any)? 
    \item[] Answer: \answerNA{} 
    \item[] Justification: No human subjects or participants involved in our research.

\item {\bf Institutional review board (IRB) approvals or equivalent for research with human subjects}
    \item[] Question: Does the paper describe potential risks incurred by study participants, whether such risks were disclosed to the subjects, and whether Institutional Review Board (IRB) approvals (or an equivalent approval/review based on the requirements of your country or institution) were obtained?
        \item[] Answer: \answerNA{} 
    \item[] Justification: The paper does not involve crowdsourcing nor research with human subjects.


\item {\bf Declaration of LLM usage}
    \item[] Question: Does the paper describe the usage of LLMs if it is an important, original, or non-standard component of the core methods in this research? Note that if the LLM is used only for writing, editing, or formatting purposes and does \emph{not} impact the core methodology, scientific rigor, or originality of the research, declaration is not required.
    \item[] Answer: \answerNA{}
    \item[] Justification: The core method development does not use LLMs as an important, original, or non-standard component. LLaMA-family models are used only as experimental backbone models, with details provided in Section~\ref{experiment} and Appendix~\ref{app:exp_set}.
 

\end{enumerate}

\end{document}